\documentclass[11pt]{jmlr} 

\usepackage{amsfonts}
\usepackage{bm}
\usepackage{bbm}

\usepackage{color} 
\usepackage{float}
\usepackage{algorithm}
\usepackage{algpseudocode}
\usepackage{enumitem}
\usepackage{comment}

\usepackage{thmtools, thm-restate}

\DeclareBoldMathCommand{\p}{p}
\DeclareBoldMathCommand{\z}{z}
\DeclareBoldMathCommand{\v}{v}
\DeclareBoldMathCommand{\w}{w}
\DeclareBoldMathCommand{\W}{W}
\DeclareBoldMathCommand{\y}{y}
\DeclareBoldMathCommand{\h}{h}
\DeclareBoldMathCommand{\q}{q}
\DeclareBoldMathCommand{\e}{e}
\DeclareBoldMathCommand{\b}{b}
\DeclareBoldMathCommand{\g}{g}
\DeclareBoldMathCommand{\u}{u}
\DeclareBoldMathCommand{\U}{U}
\DeclareBoldMathCommand{\L}{L}
\DeclareBoldMathCommand{\l}{l}
\DeclareBoldMathCommand{\1}{1}
\DeclareBoldMathCommand{\0}{0}
\DeclareBoldMathCommand{\a}{a}
\DeclareBoldMathCommand{\x}{x}
\DeclareBoldMathCommand{\c}{c}
\DeclareBoldMathCommand{\I}{I}
\DeclareBoldMathCommand{\bell}{\ell}

\renewcommand{\hat}{\widehat}
\renewcommand{\tilde}{\widetilde}
\newcommand{\regret}{\mathcal{R}}
\newcommand{\sregret}{\tilde{\mathcal{R}}}

\newcommand\Eb[1]{\E\left[ #1\right]}
\newcommand{\reals}{\mathbb{R}}

\newcommand{\half}{\tfrac{1}{2}}
\newcommand{\id}{\mathbbm{1}}

\newcommand{\yhat}{\hat{y}}
\newcommand{\yhatt}{\hat{y}_t}
\newcommand{\ystar}{y^\star}
\newcommand{\ystart}{y^\star_t}
\newcommand{\domainw}{\mathcal{W}}
\renewcommand{\ln}{\log}

\newcommand{\sign}{\textnormal{sign}}
\newcommand{\fhat}{\hat{f}}

\newcommand{\sumK}{\sum_{i=1}^K}

\newcommand{\sumT}{\sum_{t=1}^T}
\newcommand{\sumTprime}{\sum_{t=1}^{S}}
\newcommand{\sumTnolim}{\sum\nolimits_{t=1}^T}

\newcommand{\sumt}{\sum_{s=1}^t}

\newcommand\inner[2]{\langle #1, #2 \rangle}

\newcommand{\Sset}{\mathcal{S}}

\newcommand{\Fset}{\mathcal{F}}
\newcommand{\Lset}{R}

\newcommand{\Lcal}{R}

\DeclareMathOperator{\E}{\mathbb E}

\DeclareMathOperator*{\argmin}{argmin}

\SetKwProg{Init}{Initialize }{}{}
\SetKwProg{Input}{Input }{}{}
\SetKwProg{Output}{Output }{}{}

\DeclareRobustCommand{\VAN}[3]{#2} 

\newcommand{\TODO}[1]{%
\ifmmode
\text{\textcolor{red}{TODO: #1}}
\else
\textcolor{red}{TODO: #1}
\fi
}

\title[Regret-Variance Trade-Off in Online Learning]{A Regret-Variance Trade-Off in Online Learning}

\author{\Name{Dirk van der Hoeven} \Email{dirk@dirkvanderhoeven.com}\\
 \addr Department of Computer Science, Universit\`a degli Studi di Milano, Italy
 \AND
 \Name{Nikita Zhivotovskiy
} \Email{nikita.zhivotovskii@math.ethz.ch}\\
 \addr Department of Mathematics, ETH Z\"{u}rich, Switzerland
 \AND
 \Name{Nicol\`o Cesa-Bianchi
} \Email{nicolo.cesa-bianchi@unimi.it}\\
 \addr Department of Computer Science, Universit\`a degli Studi di Milano, Italy
}

\begin{document}

\maketitle

\begin{abstract}%
We consider prediction with expert advice for strongly convex and bounded losses, and investigate trade-offs between regret and ``variance'' (i.e., squared difference of learner's predictions and best expert predictions).
With $K$ experts, the Exponentially Weighted Average (EWA) algorithm is known to achieve $O(\log K)$ regret.
We prove that a variant of EWA either achieves a \textsl{negative} regret (i.e., the algorithm outperforms the best expert), or guarantees a $O(\log K)$ bound on \textsl{both} variance and regret.
Building on this result, we show several examples of how variance of predictions can be exploited in learning.
In the online to batch analysis, we show that a large empirical variance allows to stop the online to batch conversion early and outperform the risk of the best predictor in the class. We also recover the optimal rate of model selection aggregation when we do not consider early stopping.
In online prediction with corrupted losses, we show that the effect of corruption on the regret can be compensated by a large variance.
In online selective sampling, we design an algorithm that samples less when the variance is large, while guaranteeing the optimal regret bound in expectation.
In online learning with abstention, we use a similar term as the variance to derive the first high-probability $O(\log K)$ regret bound in this setting.
Finally, 
we extend our results to the setting of online linear regression.
\end{abstract}


\section{Introduction}

In the online learning protocol, the learner interacts with an unknown environment in a sequence of rounds. In each round $t=1,2,\ldots, T$ the learner makes a prediction $\yhat_t \in [-\half M, \half M]$ and suffers loss $\ell_t(\yhatt)$, where $\ell_1,\ell_2\ldots$ is a sequence of differentiable loss functions unknown to the learner. The goal is to achieve small regret $\regret_T$, defined by
\begin{align*}
    \regret_T = \sumT \big(\ell_t(\yhatt) - \ell_t(\ystart)\big)
\end{align*}
where $\ystart$ are the predictions of some reference forecaster.
We consider two special cases of online learning.
In prediction with expert advice, the learner receives the predictions $y_t(1),\ldots,y_K(t) \in [-\half M, \half M]$ of $K$ experts at the beginning of each round $t$. The learner then predicts with a convex combination $\yhatt = \sumK p_t(i)y_t(i)$ of the experts' predictions, where $\p_t = \big(p_t(1),\ldots,p_t(K)\big)$ is a probability distribution over experts updated after each round. The goal in this setting is to have small regret with respect to the predictions $\ystart = y_t(i^\star)$ of any fixed expert $i^\star$.
In online linear regression, the learner has access to a feature vector $\x_t \in \reals^d$ at each round. This is used to compute predictions $\yhatt = \inner{\w_t}{\x_t}$, where $\w_t \in \reals^d$ is a parameter vector to be updated at the end of the round. The goal is to have small regret with respect to the predictions $\ystart = \inner{\u}{\x_t}$ of any fixed linear forecaster $\u \in \reals^d$ such that $\|\u\|_2 \leq D$ and $|\ystart|  \leq \half M$ for all $t$. 

In both settings, we assume that the losses $\ell_t$ are $\mu$-strongly convex, which is to say that there exists $\mu > 0$ such that for all $y, x \in [-\half M, \half M]$,
\begin{equation}\label{eq:strongconvexity}
    \ell_t(y) - \ell_t(x)  \leq (y - x) \ell_t'(y) - \frac{\mu}{2}(x - y)^2 \qquad t=1, \ldots, T~,
\end{equation}
where we write $\ell_t^{\prime}$ to denote the derivative of $\ell_t$. Without loss of generality, throughout the paper we assume that $\mu \leq 2$. 

A typical example of a strongly convex loss is the squared loss $\ell_t(y) = (y - y_t)^2$ with $\mu = 2$. While standard online learning algorithms, such as, for example, Exponentially Weighted Average (EWA) \citep{vovk1990aggregating, littlestone1994weighted}, are directly applied to the losses $\ell_t$, we take a different approach. Our algorithms provide bounds on the linearized regret $\sregret_T$ of the form 
\begin{equation}\label{eq:blackbox}
    \sregret_T = \sumT  (\yhatt - \ystart) \ell_t'(\yhatt) \leq \frac{C_T}{\eta} + B_T + \eta \sumT (\yhatt - \ystart)^2.
\end{equation}
for some $\eta \in (0, H]$ to be chosen freely by the learner and where $H, C_T, B_T \geq 0$ are problem-specific parameters. In the expert setting, several algorithms provide such a guarantee: Squint \citep{koolen2015second}, Adapt-ML-Prod \citep{gaillard2014second}, BOA \citep{wintenberger2017optimal}, Squint+C and Squint+L \citep{mhammedi2019lipschitz}. In online linear regression, the MetaGrad algorithm \citep{VanErven2021metagrad} and its variants by \citet{mhammedi2019lipschitz, wang2020adaptivity, chen2021impossible} satisfy bounds like \eqref{eq:blackbox}.
Although in both settings the aforementioned algorithms achieve the optimal tuning of $\eta$ in~\eqref{eq:blackbox} without any preliminary information, our applications do not require this tuning property. As a consequence, the algorithms we derive in Section~\ref{sec:algorithms} are simpler and have slightly better guarantees. More importantly, unlike the aforementioned algorithms, we focus on exploiting the curvature of the loss. In particular, by combining equations~\eqref{eq:strongconvexity} and~\eqref{eq:blackbox}, we obtain the following lemma, which features the central inequality of our work.
\begin{lemma}\label{lem:moVarlessR}
Consider a sequence $\ell_1,\ldots,\ell_T$ of $\mu$-strongly convex differentiable losses.
Suppose that predictions $\yhatt$ guarantee the second-order bound \eqref{eq:blackbox}. Then 
\begin{align*}
    \regret_T \leq \frac{C_T}{\eta} + B_T - \left(\frac{\mu}{2} - \eta\right) \sumT (\yhatt - \ystart)^2.
\end{align*}
\end{lemma}
Next, we show a prototypical example of the bounds we derive in the rest of this work.
\begin{example}\label{ex:intro}
Consider the expert setting with the squared loss, $\ell_t(y) = (y - y_t)^2$, where $y, y_t \in [-1, 1]$ for all $t \ge 1$. Then the predictions $\yhatt$ of our algorithm satisfy
\begin{align}\label{eq:motivatingbound}
    \regret_T \leq 32\ln(K) - \frac{1}{2}\sumT (\yhatt - \ystart)^2~.
\end{align}
\end{example} 
Even though~\eqref{eq:motivatingbound} seems relatively inconsequential, we exploit the negative variance term in several applications. A straightforward implication is that \eqref{eq:motivatingbound} recovers the usual \emph{constant} regret bound $\regret_T = O(\ln(K))$. More interestingly, however: if we attain the worst-case performance $O(\ln(K))$, then the variance is small and bounded by $O(\ln(K))$, if the variance is large enough, then the regret becomes negative. For the role of negative variance terms in the analysis of statistical and online learning, we refer to Section \ref{sec:relwork}. 

\begin{remark}
Negative quadratic terms similar to the one in Lemma~\ref{lem:moVarlessR} appear in online convex optimization, for example in the analysis of the online gradient descent \citep[Theorem 3.3]{hazan2016introduction}). However, as far as we are aware, in online convex optimization the algorithms are not tuned to obtain a negative term in the regret bounds. In this paper, we argue that for online linear regression and for prediction with expert advice, setting $\eta < \mu/2$ in Lemma~\ref{lem:moVarlessR} and thus obtaining a negative term can be an important tool, which plays a central role in all of our applications. In prediction with expert advice the standard approach to exploit the curvature of the loss is through mixability rather than through the negative quadratic terms. Summing up, we show that the regret bound in~\eqref{eq:motivatingbound} cannot be achieved by EWA with a fixed learning rate, despite the fact that constant regret is achievable by this algorithm via mixability.

\end{remark}

\begin{proposition}[Informal]
\label{prop:informalabstract} Consider the setup of online prediction with expert advice and bounded strongly convex losses.
The bound~\eqref{eq:motivatingbound} cannot be achieved by the standard EWA algorithm.
\end{proposition}
For the sake of presentation, we defer the formal description of this result to Section~\ref{sec:alg for PwEA}. 
The proof of Proposition~\ref{prop:informalabstract} is motivated by a result of \cite{audibert2007progressive} on the sub-optimality of online to batch converted EWA in deviation. 

\paragraph{Contributions and Outline.} In Section~\ref{sec:statisticallearning}, we show our first three applications of Lemma~\ref{lem:moVarlessR} in the framework of statistical learning. In the expert setting, we show that online to batch conversion may be stopped early if the empirical variance of our predictions is sufficiently large. In particular, in high-variance regimes we may stop early because the excess risk bound is negative with high probability, and when the variance is small, we recover the optimal excess risk bound up to a $\log \log T$ factor. 
By exploiting the negative variance term again, we show an optimal high-probability excess risk bound for online to batch conversion of algorithms that satisfy \eqref{eq:blackbox}. The optimal high-probability excess risk bound (called the optimal rate of model selection aggregation) was previously known to be achieved by several estimators appearing in \citep{audibert2007progressive, lecue2009aggregation, lecue2014optimal} whose analyses are specific to the statistical learning setup. Our result can also be seen as a simplification and strengthening of a result by \citet{wintenberger2017optimal}. 
We also show a high-probability excess risk bound for online to batch conversion of an online regression algorithm in the bounded setup, which is to say that both the feature vectors and derivatives of the losses are bounded. It was previously shown by \cite{mourtada2021distribution} that online to batch converted versions of the optimal Vovk-Azoury-Warmuth forecaster \citep{vovk2001competitive, azoury2001relative} have constant excess risk with constant probability.

In Section~\ref{sec:corruption}, we use the negative variance term to counteract corrupted feedback in online learning. Our result complements previous results where losses are assumed to be stochastic and corrupted by an adversary---see, e.g., \citep{lykouris2018stochastic, Zimmert2021tsallis, amir2020prediction, ito2021optimal}. 

In Section~\ref{sec:selectivesample} we consider the selective sampling setting \citep{atlas1990training}, where the learner's goal is to control regret while saving on the number of time the current loss is observed. We show that the optimal bound can be recovered while only observing a fraction of all losses if losses are observed proportionally to the cumulative variance.

Finally, in Section~\ref{sec:biggerpicture} we discuss how our ideas are not limited to online learning with strongly convex losses, but may also be applied to online learning with abstention and to online multiclass classification. 

All bounds in the main text with suppressed constants have detailed statements in the appendix. Before discussing our applications, we introduce some notation and discuss the related work. The algorithms that we use to derive most of our results are introduced in Section~\ref{sec:algorithms}. All proofs are deferred to the Appendix.

\paragraph{Notation.}  
We use the standard $O(\cdot)$ notation. We use $\ln(\cdot)$ to denote the logarithm with base $e$. The symbol $\id[A]$ denotes the indicator of the event $A$. For an integer $K$ we denote $[K] = \{1, \ldots, K\}$. 
We use $p(i) \propto g(i)$ to denote $p(i) = g(i)\left/\left(\sum_{i' = 1}^K g(i')\right)\right.$ and $\partial_y \ell(y, Y)$ to denote the derivative of $\ell$ with respect to $y$. The symbol $I$ denotes the identity matrix whose dimensions are clear from the context. For a set of random variables $Y, X_1, \ldots, X_{t - 1}$ we write $\E_{t-1}[Y] = \E[Y|X_1, \ldots, X_{t-1}]$. 

\section{Related work}
\label{sec:relwork}
\paragraph{Online learning with losses with curvature.}

For a thorough introduction to online learning, we refer the reader to \citet{cesa2006prediction, hazan2016introduction, orabona2019modern}. Strongly convex losses are a special case of mixable losses. With mixable losses, EWA on a finite set of experts achieves $O(\log(K))$ regret---see, e.g., \citep{cesa2006prediction, mhammedi2018constant}. \citet{mhammedi2018constant} observe that---in some cases---regret may be negative for mixable losses, but they do not explore this topic further. For a description of online convex optimization, which is different from the expert setting we are focusing on, we refer to the textbook \citep{hazan2016introduction}.

\paragraph{The role of the negative terms due to curvature in online and statistical learning.}

Negative terms typically appear when proving fast rates in statistical learning with squared loss. In particular, the \emph{empirical star} algorithm of \citet{audibert2007progressive}---as well as other aggregation algorithms \citep{lecue2009aggregation, lecue2014optimal, wintenberger2017optimal}---exploit the curvature of the loss through the negative term which compensates the variance term (see \citep{kanade2022exponential} for a detailed discussion in the context of statistical learning). Similarly, in the context of online learning the negative quadratic term appears in \citep{rakhlin2014online}, where the so-called \emph{sequential offset Rademacher complexity} is studied. \citet{VanErven2021metagrad} also obtain a bound that is very similar to the bound in Lemma~\ref{lem:moVarlessR} (specifically in the proof of their Theorem 1), but they choose $\eta$ to match the negative term in their bound.  Importantly, in these papers the role of the negative term is only to get the fast rate by compensating the variance term. In contrast, in our case the negative variance terms also appear in the final regret bound and play their role in applications.

\paragraph{Suboptimality of EWA for prediction with expert advice.}
In the setup of prediction with expert advice, the classical Exponentially Weighted Average (EWA) algorithm \citep{vovk1990aggregating, littlestone1994weighted} is known to give a constant regret in the case of strongly convex losses. However, despite being optimal in this setting, this algorithm has several known drawbacks:
\begin{itemize}[topsep=0pt,parsep=0pt,itemsep=0pt]
    \item In the case of general losses, EWA does not deliver the second-order bound~\eqref{eq:blackbox}. As a result, unlike more advanced algorithms such as Squint \citep{koolen2015second}, EWA with a fixed learning rate (and even a decreasing learning rate) cannot adapt to certain benign stochastic environments where, for example, the Bernstein assumption holds \citep{mourtada2019optimality}. 
    \item The second source of suboptimality comes from online to batch conversions in the strongly convex case. Although in the statistical setting EWA performs optimally in expectation, it does not do so with high probability \citep{audibert2007progressive}. As a matter of fact, it will be clear from our analysis that this is related to the fact that EWA does not satisfy a bound of the form~\eqref{eq:motivatingbound} (see Theorem~\ref{prop:suboptimlitymain}). This problem is one of the motivations behind the work of \citet{wintenberger2017optimal}.
\end{itemize}
\paragraph{Exploiting negative regret.}
The possibility of getting a negative excess risk has been recently explicitly exploited by \citet{puchkin21a} in the setup of active learning with abstentions. 
In the context of online to batch conversion of online learning algorithms, a similar idea is exploited in Section~\ref{sec:earlystopping}. Moreover, our selective sampling results in Section~\ref{sec:selectivesample} are of the same flavor. 

\section{Our Algorithms}\label{sec:algorithms}
In the following, $g_1,g_2,\ldots$ are real numbers in a bounded interval. However, in most applications we use $g_t = \ell_t'(\yhatt)$. The algorithms in this section are simplified versions of algorithms in the literature. Namely, in Section~\ref{sec:alg for PwEA} we present a simplified version of Squint \citep{koolen2015second} and in Section~\ref{sec:alg for regression} we present a simplified version of MetaGrad \citep{VanErven2021metagrad}. The simplifications lie in the tuning of the learning rate. Squint and MetaGrad optimize the learning rate online at a small cost, whereas in our applications we only need a fixed learning rate that is known in advance, in which case we do not pay the small cost for optimizing the learning rate. The proofs of the results in this section are postponed to Appendix~\ref{app:algorithms}.

\subsection{A Simple Algorithm for Prediction with Expert Advice}\label{sec:alg for PwEA}

\begin{algorithm}[t]
\caption{An algorithm for prediction with expert advice}\label{alg:PwEA}
\Input{$\eta \in (0, \tfrac{1}{2}]$, $M > 0$} \;
\Init{$\gamma = \frac{\eta}{M^2}$, $p_1(i) = \frac{1}{K}$ $\textnormal{for all}$ $i$} \;
\For{$t = 1, \ldots, T$}{
    Receive expert predictions $y_t(1), \ldots, y_t(K)$ \\
    Predict $\yhatt = \sumK p_t(i)y_t(i)$ \\
    Receive $g_t$ and $\kappa_t$\\
    Set $\tilde{\ell}_t(i) = \gamma (y_t(i) - \yhatt)g_t + \kappa_{t-1}(\gamma (y_t(i) - \yhatt)g_t)^2$ \\ 
    Set $p_{t + 1}(i) \propto \exp(-\kappa_t\sumt \tilde{\ell}_s(i))$ 
}
\end{algorithm}

Algorithm~\ref{alg:PwEA} is a simplified version of Squint \citep{koolen2015second}. The $\kappa_t$ parameter is relevant only to the selective sampling setting, where it is used to control the range of loss estimates. In all other settings we set $\kappa_t = 1$ for all $t$. Note that in round $t$ both $\kappa_t$ and $\kappa_{t-1}$ are used to update the algorithm. This is due to a technicality in the analysis of the algorithm, where a $\kappa_{t-1}\tilde{\ell}_t(i)$ term appears and we want to use the inequality $x-x^2 \leq \ln(1 + x)$ for $|x| \leq \half$ (specifically in equation~\eqref{eq:pweaEWbound}). The regret bound of Algorithm~\ref{alg:PwEA} can be found in Lemma~\ref{th:simplePwEAregret}.
\begin{restatable}{relemma}{thsimplepwearegret}\label{th:simplePwEAregret}
For all $g_1,\ldots,g_T \in [-M,M]$ and $\kappa_0 = \kappa_1 \geq \kappa_2 \cdots \geq \kappa_T$ such that  $\kappa_t \in (0, 1]$, the predictions $\yhatt$ of Algorithm~\ref{alg:PwEA} run with input $M$ and $\eta \in (0, \half]$ satisfy 
\begin{align*}
    \sumT (\yhatt - y_t(i) )g_t 
    \leq & \frac{M^2\ln(K)}{\kappa_T \eta} + \eta \sumT \kappa_{t-1} (\yhatt - y_t^\star)^2~,
\end{align*}
provided $\max_i|y_t(i) - y_t^\star| \leq M$ for all $t \ge 1$.
\end{restatable}
As an immediate corollary of Lemma~\ref{th:simplePwEAregret} and Lemma~\ref{lem:moVarlessR} we have the following regret bound.
\begin{corollary}
\label{cor:negativeterm}
Fix an arbitrary sequence $\ell_1, \ldots, \ell_T$ of $\mu$-strongly convex differentiable losses such that $\max_t|\ell_t^{\prime}| \leq M$. Provided that $\max_i|\ystart - y_t(i)| \leq M$ for all $t \ge 1$, the predictions $\yhatt$ of Algorithm~\ref{alg:PwEA} run with inputs $M$ and $\eta \in [0, \half)$, $\kappa_t = 1 ~\textrm{for all}~ t$, and feedback $g_t = \ell_t^{\prime}(\yhatt)$, satisfy
\[
\regret_T \leq \frac{M^2\ln(K)}{\eta} - \left(\frac{\mu}{2} - \eta\right) \sumT (\yhatt - \ystart)^2~.
\]
\end{corollary}
As an example, let us consider the squared loss, which is $2$-strongly convex. In the setup of Example~\ref{ex:intro}, since $|\ell_t'(\yhatt)| = 2|(\yhatt - y_t)| \leq 4$,  Algorithm~\ref{alg:PwEA} with $g_t = \ell_t'(\yhatt)$ and $\eta = \half$ gives us the regret bound claimed in Example~\ref{ex:intro}, namely
\begin{equation}
\label{eq:negativeregretbound}
\regret_T \leq 32\ln(K) - \frac{1}{2}\sumT (\yhatt - \ystart)^2.
\end{equation}
Our next result is a formal version of Proposition~\ref{prop:informalabstract} saying that the above regret bound cannot be achieved by the standard EWA algorithm. For standard notation and explicit details on this algorithm we refer to Appendix~\ref{sec:EW}. 
\begin{restatable}{retheorem}{propsuboptimallitymain}
\label{prop:suboptimlitymain}
Consider the squared loss and two experts $y_t(1) = 0$ and $y_t(2) = 1$ for all $t \ge 1$.
Let $\yhatt^{\textrm{\,EWA}}$ be the EWA predictions. There is a sequence $y_1,y_2,\ldots$ such that $y_t \in [0, 1]$, $t \ge 1$ and, for large enough $T$, the regret of EWA with $\eta = \half$ satisfies
$
- 12\log T \le \regret_T \le 2\ln 2
$
and, at the same time,
\[
\sumT (\yhatt^{\textrm{\,EWA}} - \ystart)^2 \ge T/2~.
\]
\end{restatable}

\subsection{A Simple Algorithm for Online Linear Regression}\label{sec:alg for regression}

\begin{algorithm}[t]
\caption{An algorithm for online linear regression}\label{alg:onlineregression}
\Input{$\eta > 0$, $\sigma > 0$, $G > 0$, $Z > 0$} \;
\Init{$\gamma = \frac{\eta}{G^2}$, $\w_1 = \0$, $\textnormal{and}$ $\Sigma_1^{-1} = \frac{1}{\sigma} I$} \;
\For{$t = 1, \ldots, T$}{
    Receive $\x_t$ \\
    Set $\domainw_{t} = \bigcap_{s = 1}^t \{\w: |\inner{\w}{\x_s}| \leq Z\}$ \\
    Set $\w_{t} = \argmin_{\w \in \domainw_{t}} (\w - \tilde{\w}_t)^\top \Sigma_t^{-1}(\w - \tilde{\w}_t)$ \\
    Predict $\yhatt = \inner{\w_t}{\x_t}$ \\
    Receive $g_t$ and $\kappa_t$ \\
    Set $\z_t = \gamma \x_t g_t$ \\
    Set $\Sigma_{t+1}^{-1} = \kappa_t 2\z_t\z_t^\top + \Sigma_{t}^{-1}$ \\
    Set $\tilde{\w}_{t+1} = \w_t - \z_t \Sigma_{t+1}$
}
\end{algorithm}

In the following we use $y_t(\u) = \inner{\u}{\x_t}$. We prove a regret bound for Algorithm~\ref{alg:onlineregression}, which is a simplified version of MetaGrad \citep{VanErven2021metagrad}.  The role of the parameter $Z$ in the algorithm is to ensure the predictions $\yhatt$ are bounded, which will be important in the statistical learning setting in Section~\ref{sec:statisticallearning}. Similarly to Algorithm~\ref{alg:PwEA}, the $\kappa_t$ parameter is only used in the selective sampling setting.
\begin{restatable}{relemma}{thsimpleregregret}\label{th:simpleregregret}
For all $g_1,\ldots,g_T\in\reals$ and $\kappa_1 \geq \cdots \geq \kappa_T \in (0, 1]$, the predictions $\yhatt$ of 
Algorithm~\ref{alg:onlineregression} run with inputs $\eta > 0$, $\sigma = D^2$, $G \ge \max_t |g_t|$, and $Z > 0$
\begin{align*}
    \sumT (\yhatt - y_t(\u))g_t \leq & \frac{dG^2}{2 \kappa_T \eta} \ln\left(1 + D^2 \eta^2\left(\max_{t=1,\ldots,T}\|\x_t\|_2^2\right)\frac{T}{d}\right) + \frac{G^2}{2\eta} +  \eta \sumT \kappa_t (\yhatt - y_t(\u))^2~,
\end{align*}
for any $\x_1,\ldots,\x_T\in\reals^d$, and for any $\u \in \domainw_T \equiv \bigcap_{t = 1}^T \{\w: |\inner{\w}{\x_t}| \leq Z\}$ such that $\|\u\|_2 \leq D$.
\end{restatable}

\begin{example}
Consider the setup of Example~\ref{ex:intro} and suppose that $\max_t\|\x_t\|_2, \|\u\|_2 \leq 1$ and $M = 1$. We have that $|\ell_t'(\yhatt)| = |2(\yhatt - y_t)| \leq 4$ and thus, by Lemma~\ref{th:simpleregregret}, an appropriately tuned Algorithm~\ref{alg:onlineregression} with $g_t = \ell_t'(\yhatt)$ satisfies~\eqref{eq:blackbox} with $C_T = 8 + 8 d \log\left(1 + \eta^2\frac{T}{d}\right)$ and $B_T = 0$. Thus, by Lemma~\ref{lem:moVarlessR}, setting $\eta = \half$ gives us
\[
    R_T \leq 16 + 16d \ln\left(1 + \tfrac{T}{4d}\right) - \half \sumTnolim (\yhatt - y_t(\u))^2.
\]
This should be compared with the bound of the Vovk-Azoury-Warmuth forecaster \citep{vovk2001competitive, azoury2001relative}, where the negative term does not appear.
\end{example}

\section{Statistical Learning}\label{sec:statisticallearning}

We discuss an application of our general results in the context of statistical learning where we are interested in the generalization of estimators to unseen samples. A tool often used in converting online learning algorithms to the statistical learning setting is online to batch conversion \citep{cesa2004generalization}. Let us recall the setup.

Assume that we are given a family $\mathcal F$ of real-valued functions defined on the instance space $\mathcal X$. We observe $T$ i.i.d.\ observations $(X_t, Y_t)_{t = 1}^T$ distributed according to some unknown distribution $\mathbb{P}$ on $\mathcal X \times \mathbb{R}$. Given the loss function $\ell: \mathbb{R}^2 \to \mathbb{R}$, define the \emph{risk} $R(f)$ of $f: \mathcal X \to \mathbb{R}$ as 
$
R(f) = \E\ell(f(X), Y)
$,
where the expectation is taken with respect to the joint distribution of $X$ and $Y$. We are interested in bounding the \emph{excess risk}
\[
R(\hat f) - \inf\limits_{f \in \mathcal F}R(f)~,
\]
where $\hat f$ is constructed based on the sample $(X_t, Y_t)_{t = 1}^T$. Assume that there is a sequence of predictors $\hat f_1, \ldots, \hat f_T$ trained in an online manner using $(X_t, Y_t)_{t = 1}^T$ (that is, $\hat f_k$ depends on $(X_t, Y_t)_{t = 1}^{k - 1}$) such that almost surely
$
\sum\nolimits_{t=1}^T\left(\ell(\hat f_{t}(X_t), Y_t) -  \ell(f^{\star}(X_t), Y_t)\right)\le R_T~,
$
where $R_T$ is non-random. In this case, a standard online to batch conversion approach gives an in-expectation excess risk bound 
\[
\E R\left(\frac{1}{T}\sumT f_t\right) - \inf\limits_{f \in \mathcal F}R(f) \le \frac{R_T}{T}~,
\]
for any loss convex in its first argument and where the expectation is taken with respect to the learning sample $(X_t, Y_t)_{t = 1}^T$. However, getting a high-probability version of this result is a known challenge if one wants to get the fast rate $O\left(\frac{1}{T}\right)$. A standard way of proving a high-probability result is to apply Freedman's inequality for martingales \citep{kakade2008generalization} that leads in the worst case to a variance term scaling as $O\left(\frac{1}{\sqrt{T}}\right)$. For example, \cite{audibert2007progressive} showed that this is the case if one wants to prove a high-probability excess risk bounds based on EWA. A way to handle the variance term in Freedman's inequality is by exploiting the Bernstein assumption as in \citep{kakade2008generalization}. Unfortunately, this assumption is not necessarily satisfied by the stochastic environments we are considering. The main idea in this section is to use the negative term from Lemma~\ref{lem:moVarlessR} to cancel out this variance term appearing due to Freedman's inequality\footnote{We remark that \citet{wintenberger2017optimal} uses a similar but technically more involved idea to compensate the variance of predictions using the term appearing because of the curvature of the loss.}.

We use the following notation when applying the online algorithms in the statistical setting:
\[
\ell_t(\cdot) = \ell(\cdot, Y_t), \quad \textrm{and}\quad y_t(f) = f(X_t)~.
\]

\subsection{Statistical Learning: Model Selection Aggregation}\label{sec:earlystopping}

In this section, we discuss the application of our results to the model selection (MS) aggregation. This setup was introduced by \citet{nemirovski2000topics} and further studied by \citet{tsybakov2003optimal} and by \citet{audibert2007progressive, lecue2009aggregation, lecue2014optimal, wintenberger2017optimal, mourtada2021distribution} among other works. In this setup, we are given a finite dictionary $\mathcal F = \{f_1, \ldots, f_K\}$ of real-valued absolutely bounded functions. In the model selection  aggregation, one is interested in constructing an estimator $\widehat{f}$ based on the random sample $(X_t, Y_t)_{t = 1}^T$ such that, with probability at least $1 - \delta$,
\begin{equation}
\label{eq:optrateofaggregation}
R(\widehat{f}) - \min\limits_{f \in \mathcal F}R(f) = O\left(\frac{\log(K) + \log(1/\delta)}{T}\right)~,
\end{equation}
under appropriate boundedness and curvature assumptions on the loss function $\ell$. Analogously to \citep{tsybakov2003optimal}, the bound of the form~\eqref{eq:optrateofaggregation} will be called the \emph{optimal rate of aggregation}. 
We make use of a variant of online to batch conversion \citep{cesa2004generalization} where we stop the procedure early if the empirical variance of predictions is sufficiently large. We sketch the idea. Let $S$ be the number of samples we have used before we terminated the procedure. We use Algorithm~\ref{alg:PwEA} as our aggregation procedure and use $\fhat = \frac{1}{S}\sum_{t = 1}^S \sumK p_t(i)f_i$. By Jensen's inequality we have 
\begin{equation*}
    R(\fhat) \leq \frac{1}{S}\sumTprime \E_{t-1}\left[\ell_t\left(\sumK p_t(i)f_i(X_t)\right)\right].
\end{equation*}
To motivate stopping early, observe that if the empirical variance in Lemma~\ref{lem:moVarlessR} is sufficiently large, we may conclude that the excess risk is negative and we have outperformed the best $f \in \Fset$. The result can be found in Theorem~\ref{prop:expertearlystopbigO} below, whose proof is implied by Theorem~\ref{th:expertbatch} in Appendix~\ref{app:statlearn}.

\begin{algorithm}[t]
\caption{Early Stopping online to batch for Model Selection Aggregation}\label{alg:expertsonlinetobatch}
\Input{$T$, $M$, $\eta$, stopping threshold $\Sset$} \;
\Init{$S = 0$, provide $\eta$, and $M$ as input for Algorithm~\ref{alg:PwEA}} \;
\While{$S < T$ and $\Sset > \frac{\mu}{8}\min_{f \in \Fset}\sumTprime(\yhatt - f(X_t))^2$}{ 
    Receive $X_t$ and send $f_1(X_t), \ldots, f_K(X_t)$ as expert predictions to Algorithm~\ref{alg:PwEA} \\
    Receive $\p_t$ and $\yhatt = \sumK p_t(i)f_i(X_t)$ from Algorithm~\ref{alg:PwEA} \\
    Predict $\yhatt$ and receive $\ell_t$ \\
    Send $g_t = \ell_t'(\yhatt)$ and $\kappa_t = 1$ to Algorithm~\ref{alg:PwEA}\\
    Set $S = S + 1$
}\;
\Output{$\hat{f} = \frac{1}{S}\sumTprime \sumK p_t(i)f_i$} \;
\end{algorithm}

\begin{theorem}\label{prop:expertearlystopbigO}
Suppose that for all $f \in \Fset$ $|f(X)| \leq \half $ almost surely, that $|\partial_y \ell(y, Y)| \leq 1$ almost surely for all $y$ such that $|y| \leq \half $, and that $\ell$ is $\mu$-strongly convex in its first argument.
Then, with probability at least $1 - \delta$, Algorithm~\ref{alg:expertsonlinetobatch} with input parameters $T$, $\Sset = O\left(\frac{\ln(K) + \log(\log (T)/\delta))}{\mu}\right)$, $\eta = \frac{\mu}{8}$, and $M=1$ satisfies
\begin{equation*}
    \Lset(\fhat) \leq
    \begin{cases}
    \min_{f \in \Fset} \Lset(f) & \text{if $S < T$} \\
    \min_{f \in \Fset} \Lset(f) + O\left(\frac{\ln(K) + \log(\log (T)/\delta))}{\mu T}\right) & \text{if $S = T$,}
    \end{cases}
\end{equation*}
where $S$ is the number of steps of Algorithm \ref{alg:expertsonlinetobatch}.
\end{theorem}
When Algorithm~\ref{alg:expertsonlinetobatch} terminates at step $S = T$, we recover the optimal high probability bound for model selection aggregation \eqref{eq:optrateofaggregation} up to an additive $\log\log T$ term.  
However, when $S = T$ our bound tells us slightly more, because we know that for all $t^{\prime} < T$,  
$
\min_{f \in \Fset}\sum\nolimits_{t = 1}^{t^{\prime}}(\yhatt - f(X_t))^2 = O(\log K + \log\log T)~,
$
which means that on the sequence $(X_t, Y_t)_{t = 1}^T$ our predictions $\yhatt$ are essentially following the prediction of the currently best expert at each round.

In the special case where we are solely interested in the best possible performance of the online to batch conversion of Algorithm~\ref{alg:PwEA}, we can remove the $\log\log T$ term appearing in the previous bound. We remark that, apart from the work of \citet{wintenberger2017optimal}, no known analysis based on the online to batch conversion achieved the optimal rate of aggregation~\eqref{eq:optrateofaggregation}. We also believe that our analysis is simpler than for previously known algorithms. The result can be found in Theorem~\ref{prop:expertbatchnonUni} below, whose result is implied by Theorem~\ref{th:expertbatchnonUni} in Appendix~\ref{app:statlearn}.

\begin{theorem}\label{prop:expertbatchnonUni}
Suppose that for all $f \in \Fset$, $|f(X)| \leq \half $ almost surely, $|\partial_y \ell(y, Y)| \leq 1$ almost surely for all $y$ such that $|y| \leq \half $, and that $\ell$ is $\mu$-strongly convex in its first argument. Then, with probability at least $1 - \delta$, Algorithm~\ref{alg:expertsonlinetobatch} with input parameters $T$, $\eta = \frac{\mu}{4}$, $\Sset = \infty$, and $M=1$, guarantees
\begin{align*}
    & \Lset(\fhat) - \min_{f \in \Fset}\Lset(f)  = O\left(\frac{\ln(K) + \ln(1/\delta)}{\mu T}\right).
\end{align*}
\end{theorem}

\subsection{Statistical Learning: Linear Regression} \label{sec:regstatlearn}
We consider the statistical learning setting where one has access to $T$ i.i.d.\ samples of pairs $(X_t, Y_t) \in \reals^d \times \reals$. We consider $\mathcal{F} \subseteq \{\x \mapsto \inner{\w}{\x}: \w \in \mathbb{R}^d\}$. For $\w \in \mathbb{R}^d$ we define the risk as $\Lcal(\w) = \Eb{\ell(\inner{\w}{X}, Y)}$. As above, $\ell$ is $\mu$-strongly convex in its first argument. 

There are no known high probability excess risk bounds in linear regression based on online to batch conversions with convergence rate $O\left(\frac{d\ln(T)}{T}\right)$. We provide such a result in the bounded setup where the feature vectors, derivatives of the losses, and the norm of the reference vector are bounded. Similarly to before, for a result that holds with high probability, one needs to control the cumulative variance of our prediction. For standard online learning algorithms the control of the variance may prove troublesome. For example, \citet{mourtada2021distribution} showed that a version of Vovk-Azoury-Warmuth forecaster \citep{vovk2001competitive, azoury2001relative} may have a $O(1)$ excess risk bound with constant probability, whereas in expectation the Vovk-Azoury-Warmuth forecaster guarantees a $O\left(\frac{d\ln(T)}{T}\right)$ excess risk bound.  Instead, we leverage the negative empirical variance of Lemma~\ref{lem:moVarlessR} to control the variance of the online to batch conversion, leading to the following excess risk bound, whose result is implied by Theorem~\ref{th:otbregression} in Appendix~\ref{app:statlearn}.  
\begin{theorem}\label{prop:otbregression}
Suppose that $\|X\|_2 \leq 1$ and $\sup\nolimits_{y \in [-1, 1]}|\partial_y \ell(y, Y)| \le 1$ almost surely, $\|\w\|_2 \leq 1$, and that $\ell$ is $\mu$-strongly convex in its first argument.
Then, with probability at least $1 - \delta$,
\begin{align*}
    & \Lcal\left(\frac{1}{T}\sumT \w_t\right) - \Lcal(\w) = O\left(\frac{d\ln(T) + \ln(1/\delta)}{\mu T}\right)~,
\end{align*}
where $\w_t$ are given by Algorithm~\ref{alg:onlineregression} with $\eta = \frac{\mu}{4}$, $\sigma = 1$, $Z = 1$, $G = 1$, $\kappa_t = 1$, and feedback $g_t = \ell_t'(\inner{\w_t}{X_t})$ for $t = 1, \ldots, T$.
\end{theorem}

\section{Corrupted feedback}\label{sec:corruption}
In this section, we study a setting where the loss derivatives $\ell_t'(\yhatt)$ observed by the learner at each round $t$ may be adversarially corrupted by unknown additive constants $c_t$, and
we are interested in the best possible dependence on $c_1,\ldots,c_T$ in the regret bound. 
To better explain our setting, we start with the following example.

\begin{example}
\label{ex:corrsqloss}
In the online regression setting, suppose that $\ell_t(\yhatt) = (\yhatt - y_t)^2$ for all $t$, but the learner observes corrupted outcomes $y_t - c_t/2$.
Hence, the squared loss derivative computed by the learner is $2(\yhatt - y_t + c_t/2) = \ell_t'(\yhatt) + c_t$, which can be handled by the algorithms developed in this section. 
\end{example}
Several variants of this setting have been studied in prior work, see for example \citep{lykouris2018stochastic, amir2020prediction, Zimmert2021tsallis, ito2021optimal} and the references therein.
The main difference between our setting and these previous settings is that we assume our losses to be strongly convex and the environment is not necessarily stochastic. Although the results in this section are rather straightforward corollaries of our bounds, we believe that it is instructive to provide some explicit results. All proofs of the results in this section are postponed to Appendix~\ref{app:corruption}.

Our first result shows the performance of Algorithm~\ref{alg:PwEA} in the setup with corrupted gradients. The proof follows from observing that $\ell_t'(\yhatt)(\yhatt - \ystart) = (\ell_t'(\yhatt) + c_t)(\yhatt - \ystart) - c_t(\yhatt - \ystart)$ and that for any $\lambda > 0$, the inequality
\begin{equation}
\label{eq:elem}
    |c_t(\yhatt - \ystart)| \leq \frac{c_t^2}{\lambda} + \frac{\lambda}{4} (\yhatt - \ystart)^2
\end{equation}
holds.
The $\frac{\lambda}{4} (\yhatt - \ystart)^2$ term can be compensated for by the negative $\frac{\mu}{2}(\yhatt - \ystart)^2$ appearing in Lemma~\ref{lem:moVarlessR}, leading to a $\sumT c_t^2$ additive term in the regret bound. In particular, our result implies that as long as $\sumT c_t^2$ is of order $O(\log K)$, the same regret bound \eqref{eq:negativeregretbound} can be achieved up to constant factors as if the losses were not corrupted. The formal statement can be found in Theorem~\ref{prop:expertcorruption} below. 
\begin{restatable}{retheorem}{propexpertcorrupt}
\label{prop:expertcorruption}
Fix an arbitrary sequence $\ell_1,\ldots,\ell_T$ of $\mu$-strongly convex differentiable losses and corruptions $c_1,\ldots,c_T\in\reals$. Then the predictions $\yhatt$ of Algorithm~\ref{alg:PwEA} run with inputs $M \ge \max_t|\ell_t'(\yhatt) + c_t|$, $\eta = \frac{\mu}{4}$, feedback $g_t = \ell_t'(\yhatt) + c_t$, and $\kappa_t = 1$ satisfy
\begin{equation*}
\label{eq:eqcorruption}
    \regret_T \leq \frac{{8} M^2\ln(K)}{\mu} + \sumT \frac{c_t^2}{\mu} - {\frac{\mu}{8}\sumT (\yhatt - y_t(i^\star))^2}~,
\end{equation*}
provided that $\max_i \max_t|\yhatt - y_t(i)| \leq M$.

\end{restatable}

Next we prove an analog of Theorem~\ref{prop:expertcorruption} in the online regression setup.
\begin{restatable}{retheorem}{propregcorruption}
\label{th:regcorruption}
Fix an arbitrary sequence $\ell_1,\ldots,\ell_T$ of $\mu$-strongly convex differentiable losses and corruptions $c_1,\ldots,c_T\in\reals$. Then the predictions $\yhatt$ of Algorithm~\ref{alg:onlineregression} run with inputs $\eta = {\frac{\mu}{8}}$, $\sigma = D^2$, $G \ge \max_t|\ell_t'(\yhatt) + c_t|$, $Z > 0$, feedback $g_t = \ell_t'(\yhatt) + c_t$, and $\kappa_t = 1$, satisfy
\begin{align*} %
    \regret_T %
\le
    \frac{{4}dG^2}{\mu} \ln\left(1 + \frac{TD^2 \mu^2\max_t\|\x_t\|_2^2}{2d}\right) + \frac{{4}G^2}{\mu} + \sumT \frac{c_t^2}{\mu} -{\frac{\mu}{8}\sumT (\yhatt - y_t^\star)^2},
\end{align*}
for any $\x_1,\ldots,\x_T\in\reals^d$, and for any $\u \in \domainw_T \equiv \bigcap_{t = 1}^T \{\w: |\inner{\w}{\x_t}| \leq Z\}$ such that $\|\u\|_2 \leq D$ and $\ystart = \inner{\u}{\x_t}$ for all $t \ge 1$.
\end{restatable}

\section{Selective Sampling}\label{sec:selectivesample}
We consider a variant of the selective sampling setting---see, e.g., \citep{atlas1990training, freund1997selective, cesa2003learning, cesa2006worst, orabona2011better}---where the learner has access to the expert predictions (or, equivalently, to feature vectors), but can observe its own loss only upon request. The goal is to trade off the number of loss requests with regret guarantees. We show that if the variance is high, with only a fraction of all losses requested we obtain the same guarantee (in expectation and up to constants) as when all losses are requested. 

Let $o_t = 1$ with probability $q_t$ and $o_t = 0$ with probability $1 - q_t $. In each round, if $o_t = 1$ the loss $\ell_t$ at round $t$ is requested, and we use the loss estimator $\frac{o_t}{q_{t-1}}\ell_t$ to update. Note that this is not the importance weighted estimator, as $\E_{t-1}[\frac{o_t}{q_{t-1}}] = \frac{q_t}{q_{t-1}}$. The reason for choosing this particular loss estimator is that we have better control of the range of the loss which allows us to tune $\kappa_t$ in Algorithm~\ref{alg:PwEA} accordingly. The probability of requesting a loss is 
\begin{align}\label{eq:lossqt}
    q_t = \min \left\{1, \beta\Bigg/\sqrt{\min_i \sumt (\yhat_s - y_s(i))^2}\right\},
\end{align}
where $\beta > 0$ is chosen by the learner. Our result for the selective sampling setting can be found in Theorem~\ref{prop:expertslossefficientbigO} below, whose statement is implied by Theorem~\ref{th:expertslossefficient} in Appendix~\ref{app:selsample}. Theorem~\ref{prop:expertslossefficientbigO} implies the following: if $\beta = O(\mu^{-3/2}\ln(K))$, then with only an expected number $\sumT q_t$ of loss requests, we obtain (up to constants) the same regret guarantee as we would have obtained if we had requested all losses. With this particular choice of $\beta$, $q_t < 1$ as soon the bound in Lemma~\ref{lem:moVarlessR} becomes negative. In other words, when the variance is high we only need a fraction of the losses to recover the worst-case optimal regret bound (in expectation). A similar result can be obtained in the regression setting, see Appendix~\ref{sec:regressionselsampling}.

\begin{theorem}\label{prop:expertslossefficientbigO}
Fix an arbitrary sequence $\ell_1,\ldots,\ell_T$ of $\mu$-strongly convex differentiable losses. Provided $\max_i \max_t|\yhatt - y_t(i)| \leq 1$, the predictions $\yhatt$ of Algorithm~\ref{alg:PwEA} run with inputs $M = 1 \ge \max_t|\ell_t'(\yhatt)|$, $\eta = \frac{\mu}{4}$, feedback $g_t = \frac{o_t}{q_{t-1}} \ell_t'(\yhatt)$, and $\kappa_t = q_t$  satisfy
\begin{equation*}
\Eb{\sumT(\ell_t(\yhatt) - \ell_t(\ystart))}
    = O \left( \frac{\ln(K)}{\mu} + \frac{\ln(K)^2}{\mu^3 \beta^2}\right)~.
\end{equation*}

\end{theorem}

\section{Further Extensions}\label{sec:biggerpicture}

In Appendix~\ref{sec:freerestarts} we present another application of Lemma~\ref{lem:moVarlessR}. Namely, we show that we may restart Algorithm~\ref{alg:PwEA} for free whenever the regret becomes negative. This gives us a regret bound where we compete with a new expert after each restart, which is a stronger notion of regret than when we compete with a fixed expert in all rounds.

Our applications of Lemma~\ref{lem:moVarlessR} also naturally extend beyond online learning with strongly convex losses. Here we discuss two such extensions. The online prediction with abstention setting was introduced by \citet{neu2020fast} and proceeds as follows.  In each round $t = 1, \ldots, T$ the learner receives expert predictions $y_t(i) \in \{-1, 1\}$, $i = 1, \ldots, K$ and the learner can then either predict $\tilde{y}_t \in \{-1, 1\}$ or abstain from prediction.
If the learner predicts with $\tilde{y}_t$, the learner suffers the binary loss $\ell_t(\tilde{y}_t) = \id[\tilde{y}_t \neq y_t]$, where $y_t \in \{-1, 1\}$. If the learners abstains from prediction, the learner suffers abstention cost $\rho \in [0, \half)$. Let $a_t = 1$ for prediction and $a_t = 0$ for abstention. The total loss of the learner is therefore equal to
$
\sumT (a_t\id[\tilde{y}_t \neq y_t] + (1 - a_t) \rho).
$
Assume that the prediction strategy is random is a sense that $a_t$ are Bernoulli random variables whose means might depend on previous observations. The work of \citet{neu2020fast} shows that there is a randomized prediction strategy such that for any data generating mechanism it holds that
\begin{equation}
\label{eq:regretboundforabst}
\E \left[\sumT (a_t\id[\tilde{y}_t \neq y_t] + (1 - a_t) \rho)\right] - \sumT\id[\ystart \neq y_t] = O\left(\frac{\log K}{1 - 2\rho}\right),
\end{equation}
independently of $T$,
where the expectation is taken with respect to the randomness of $a_t$; here $\ystart = y_t(i^\star)$ and $i^\star = \argmin_i \sumT \id[y_t(i) \neq y_t]$. Although it was shown that the randomization is necessary to achieve the regret bound \eqref{eq:regretboundforabst}, it is unclear if the same regret bound can be achieved with high probability with respect to the randomization of the algorithm. In Appendix \ref{sec:abstention}, we answer this question using the techniques we developed and provide a randomized algorithm such that, with probability at least $1 - \delta$,
\[
\sumT (a_t\id[\tilde{y}_t \neq y_t] + (1 - a_t) \rho) - \sumT\id[\ystart \neq y_t] = O\left(\frac{\log K + \log(1/\delta)}{1 - 2\rho}\right).
\]
In Appendix~\ref{sec:abstention} we prove Lemma~\ref{lem:abstentionnegative}, which is the analog of Lemma~\ref{lem:moVarlessR} for the abstention setting. The equivalent of the negative term in Lemma~\ref{lem:moVarlessR} is used to compensate for the variance of the high-probability statement, which allows us to recover the above bound. This also implies that our other applications of Lemma~\ref{lem:moVarlessR} can be exported to the online learning with abstention setting. 

The second extension is in online multiclass classification. In online multiclass classification the analog of Lemma~\ref{lem:moVarlessR} can be found in \citep[Lemma~2]{VanderHoeven2021beyond}. \citet{VanderHoeven2021beyond} use their Lemma~2 to derive high-probability regret bounds for their algorithm and one could also use it to export our applications of Lemma~\ref{lem:moVarlessR} to online multiclass classification.

\acks{Nikita Zhivotovskiy is funded in part by ETH Foundations of Data Science (ETH-FDS). Dirk van der Hoeven and Nicolò Cesa-Bianchi gratefully acknowledge partial support from the MIUR PRIN grant Algorithms, Games, and Digital Markets (ALGADIMAR) and the EU Horizon 2020 ICT-48 research and innovation action under grant agreement 951847, project ELISE (European Learning and Intelligent Systems Excellence).
}

\DeclareRobustCommand{\VAN}[3]{#3} 
\bibliography{sample}
\DeclareRobustCommand{\VAN}[3]{#2}
\appendix

\section{Details of Section~\ref{sec:algorithms} (Our Algorithms)}\label{app:algorithms}

We restate Lemma~\ref{th:simplePwEAregret}, after which we prove its result. 

\thsimplepwearegret*

\begin{proof}
We start by observing that {the vector of weights} $\p_t$ is obtained by running lazy EWA with learning rate $\kappa_t$ on the (signed) surrogate losses $\tilde{\ell}_t(i) = \gamma (y_t(i) - \yhatt)g_t + \kappa_t(\gamma (y_t(i) - \yhatt)g_t)^2$. We use $\kappa_0 = \kappa_1$. Thus, by \citep[Lemma 1]{vanderhoeven2018many} we have that, for any $i^\star \in [K]$,
\begin{equation}\label{eq:pweaEWbound}
\begin{split}
    & \sumT (\E_{{i \sim \p_t}}[\tilde{\ell}_t(i)] - \tilde{\ell}_t(i^\star)) \\
    & \leq \frac{\ln(K)}{\kappa_T} + \sumT \left(\E_{{i \sim \p_t}}[\tilde{\ell}_t(i)] + \frac{1}{\kappa_{t-1}}\ln \E_{{i \sim \p_t}}\left[\exp(-\kappa_{t-1}\tilde{\ell}_t(i))\right]\right).
\end{split}
\end{equation}
Now, using that $\exp(x - x^2) \leq 1 + x$ for $|x| \leq \half$ and the fact that $|\gamma (y_t(i) - \yhatt)g_t| \leq \half$ due to our choice of $\gamma$, we find that 
\begin{align*}
    & \sumT (\E_{{i \sim \p_t}}[\tilde{\ell}_t(i)] - \tilde{\ell}_t(i^\star)) \\
    & \leq \ln(K)/{\kappa_T} + \sumT \left(\E_{{i \sim \p_t}}[\tilde{\ell}_t(i)] + \frac{1}{\kappa_{t-1}}\ln \E_{i \sim \p_t}\left[1 + \gamma \kappa_{t-1} (\yhatt - y_t(i) )g_t\right]\right) \\
    & = \ln(K)/{\kappa_T} + \sumT \E_{{i \sim \p_t}}[\tilde{\ell}_t(i)] ~,
\end{align*}
where the equality is due the fact that since $\yhatt = \E_{i \sim \p_t}[y_t(i)]$, we have that $\E_{i \sim \p_t}[\gamma (\yhatt - y_t(i) )g_t] = 0$. This also implies that $\E_{{i \sim \p_t}}[\tilde{\ell}_t(i)] = \E_{{i \sim \p_t}}[\kappa_{t-1}(\gamma (\yhatt - y_t(i) )g_t)^2]$. Thus, we may write
\begin{align*}
    & \E_{{i \sim \p_t}}[\tilde{\ell}_t(i)] - \tilde{\ell}_t(i^\star) \\
    & = \E_{{i \sim \p_t}}[\kappa_{t-1}(\gamma (\yhatt - y_t({i^\star}) )g_t)^2] + \gamma (\yhatt - y_t({i^\star}) )g_t - \kappa_{t-1}(\gamma (\yhatt - y_t({i^\star}) )g_t)^2.
\end{align*}
Combining with the above and reordering we find 
\begin{align}\label{eq:regretexpertsintermediate}
    \sumT \gamma (\yhatt - y_t({i^\star}) )g_t \leq \frac{\ln(K)}{\kappa_T} + \sumT \kappa_{t-1}(\gamma (\yhatt - y_t({i^\star}) )g_t)^2
\end{align}
After dividing both sides by $\gamma = \frac{\eta}{M^2}$, this gives us
\begin{align*}
    \sumT (\yhatt - y_t({i^\star}) )g_t \leq & \frac{M^2\ln(K)}{\kappa_T \eta} + \frac{\eta}{M^2}\sumT \kappa_{t-1}( (\yhatt - y_t({i^\star}) )g_t)^2 \\
    \leq & \frac{M^2\ln(K)}{\kappa_T \eta} + \eta \sumT \kappa_{t-1}(\yhatt - y_t(i^\star))^2~.
\end{align*}
completing the proof. 
\end{proof}

We now restate Theorem~\ref{prop:suboptimlitymain} and provide its proof. 

\propsuboptimallitymain*

\begin{proof}
The proof uses a construction similar to one used in \citep{audibert2007progressive}. Our idea is to show that for some environments the output of EWA is close to the \emph{follow the leader} prediction. This can lead to large variance when we follow a wrong expert for most of the rounds. For $T$ large enough consider the following sequence:
\[
y_t =
\begin{cases}
 3/4, & \mbox{if } t \le 4\lceil\log T\rceil,
 \\ 1/2, & \mbox{if } 4\lceil\log T\rceil < t < T - 8\lceil\log T\rceil,
 \\ 1/4, & \mbox{if } T - 8\lceil\log T\rceil \le  t \le T.
 \end{cases}
\]
Fix $\eta = 1/2$. Since $y_t, y_t(1), y_t(2) \in [0, 1]$, and the squared loss is $1/2$-exp-concave on this domain (see Appendix~\ref{sec:EW}) we have 
$\regret_T \le 2\ln 2$.

Next, we show the lower bound. Since $y_t = 1/4$ appears more frequently in the sequence, we have that $y_t(1) = 0$ is the prediction of the best expert. However, until the last $8\lceil\log T\rceil$ rounds, that is, for any $4\lceil\log T\rceil < t^{\prime} < T - 8\lceil\log T\rceil$ the EWA algorithm puts most of its weight on the second expert predicting $y_t(2) = 1$. At the same time, both experts suffer the same loss when $y_t = 1/2$.
Formally, for any such $t^{\prime}$ we have
\[
\sum\limits_{t = 1}^{t^{\prime}}(y_t - y_t(1))^2 - \sum\limits_{t = 1}^{t^{\prime}}(y_t - y_t(2))^2 = 4\lceil\log T\rceil(9/16 - 1/16) = 2\lceil\log T\rceil.
\]
Therefore, for the same $t^{\prime}$, the weight of the first expert in the EWA prediction with $\eta = 1/2$ is
\begin{align*}
p_{t^{\prime}}(1) &= \frac{\exp\left(-\sum\limits_{t = 1}^{t^{\prime}}(y_t - y_t(1))^2/2\right)}{ \sum\limits_{i = 1}^2\exp\left(-\sum\limits_{t = 1}^{t^{\prime}}(y_t - y_t(i))^2/2\right)} = \frac{1}{1 + \exp(\lceil\log T\rceil)} \le \frac{1}{1 + T}~.
\end{align*}
Thus, we have 
\[
\hat{y}_{t^{\prime}}^{\textrm{\,EWA}} = p_{t^{\prime}}(1)y_{t^{\prime}}(1) + p_{t^{\prime}}(2)y_{t^{\prime}}(2)  \ge \frac{T}{T + 1}~.
\]
We are ready to bound the regret. Our idea will be just to use the boundedness of the loss when $y_t \in \{1/4, 3/4\}$ and compute the regret over remaining rounds. Using elementary algebra, we have
\begin{align*}
\regret_T &\ge (T - 12\lceil\log T\rceil)\cdot\left(\left(\frac{T}{T + 1} - 1/2\right)^2 - (1/2)^2\right) - 12\lceil\log T\rceil(3/4)^2
\\
&\ge \frac{-T(T - 12\lceil\log T\rceil)}{(T + 1)^2} - 12(3/4)^2\lceil\log T\rceil \ge -12\log T~,
\end{align*}
for all $T \ge 4$. At the same time, the following variance bound holds
\[
\sumT (\hat{y}_{t}^{\textrm{\,EWA}} - y_t(1))^2 \ge \sum\limits_{t = 4\lceil\log T\rceil + 1}^{T - 8\lceil\log T\rceil - 1}(\hat{y}_{t}^{\textrm{\,EWA}} - y_t(1))^2 \ge (T - 12\lceil\log T\rceil - 2)\left(\frac{T}{T + 1}\right)^2 \ge \frac{T}{2}~,
\]
provided that $T > 150$. The claim follows.
\end{proof}

Here we restate Lemma~\ref{th:simpleregregret}, after which we prove it.

\thsimpleregregret*

\begin{proof}
We start by observing that $\w_t$ is the mean of continuous exponential weights with a Gaussian  prior and learning rate 1 
on (signed) surrogate losses $\tilde{\ell}_t^{\textnormal{or}}(\w) = \inner{\w - \w_t}{\z_t} + \kappa_t (\w - \w_t)^\top\z_t\z_t^\top(\w - \w_t)$, see \citep[Section 4]{vanderhoeven2018many}. Thus, for any $\u \in \domainw_T$, by \citet[Theorem 5]{vanderhoeven2018many} we have that 
\begin{align*}
    \sumT \big(\tilde{\ell}_t^{\textnormal{or}}(\w_t) - \tilde{\ell}_t^{\textnormal{or}}(\u)\big) \leq & \frac{\|\u\|_2^2}{2\sigma} + \half \sumT \z_t^\top \Sigma_{t+1} \z_t~.
\end{align*}
Using that $\kappa_1 \geq \kappa_2 \ge \cdots \geq \kappa_T \in (0, 1]$ and the Sherman-Morrison formula to compute the inverse we find that
\begin{align*}
    2\z_t^\top \left(2\sumt \kappa_s \z_s\z_s^\top + \frac{1}{\sigma}I\right)^{-1} \z_t \leq & \frac{1}{\kappa_T}2\z_t^\top \left(2\sumt \z_s\z_s^\top + \frac{1}{\sigma}I\right)^{-1} \z_t \\
    = & \frac{1}{\kappa_T}\left(2\z_t^\top \Sigma_{t} \z_t - \frac{(2\z_t^\top \Sigma_{t} \z_t)^2}{1 + 2\z_t^\top \Sigma_{t} \z_t}\right) \\
    = & \frac{1}{\kappa_T}\left(1 - \frac{1}{1 + 2\z_t^\top \Sigma_{t} \z_t}\right) \\
    \leq & \frac{1}{\kappa_T} \ln\left(1 + 2\z_t^\top \Sigma_{t} \z_t \right) = \frac{1}{\kappa_T} \ln\left(\frac{\textnormal{Det}(\Sigma_{t+1})}{\textnormal{Det}(\Sigma_t)}\right),
\end{align*}
where the second inequality is due to the fact that $1 - 1/x \leq \ln(x)$ for $x > 0$ and the final equality can be found on, for example, in \citep[page 475]{meyer2000matrix}. Thus, we have that 
\begin{align*}
    \sumT \z_t^\top \Sigma_{t+1} \z_t & \leq \frac{1}{\kappa_T} \sumT \half \ln\left(\frac{\textnormal{Det}(\Sigma_{t+1})}{\textnormal{Det}(\Sigma_t)}\right) \\
    & \leq \frac{1}{2 \kappa_T} \ln\left(\frac{\textnormal{Det}(\Sigma_{T+1})}{\textnormal{Det}(\Sigma_1)}\right) = \frac{1}{2 \kappa_T} \ln\textnormal{Det}\left(I + \sigma\sumT 2 \z_t\z_t^2\right).
\end{align*}
Now, following the proof and discussion of \citet[Theorem 11.2]{cesa2006prediction} we have that 
\begin{align*}
    \ln\textnormal{Det}\left(I + \sigma^2\sumT 2 \z_t\z_t^2\right) &\leq d \ln\left(1 + 2\sigma\Big(\max_t\|\z_t\|_2^2\Big)\frac{T}{d}\right)
    \\
    &=  d \ln\left(1 + 2\sigma\gamma^2\Big(\max_t\|\x_tg_t\|_2^2\Big)\frac{T}{d}\right)~.
\end{align*}
By combining the above we find 
\begin{equation}\label{eq:regregretinter}
\begin{split}
    \sumT \inner{\w_t - \u}{\x_tg_t}  &=  \frac{1}{\gamma} \sumT (\inner{\w_t - \u}{\z_t} - \kappa_t \inner{\w_t - \u}{\z_t}^2)  + \sumT \gamma \kappa_t (\inner{\w_t - \u}{\x_tg_t})^2 \\
    &\leq  \frac{\|\u\|_2^2}{2\sigma\gamma} + \frac{d}{2\gamma\kappa_T} \ln\left(1 + 2\sigma\gamma^2\Big(\max_t\|\x_tg_t\|_2^2\Big)\frac{T}{d}\right) 
    \\
    &\qquad+ \sumT \gamma \kappa_t (\inner{\w_t - \u}{\x_tg_t})^2.
\end{split}
\end{equation}
Now, we continue by using that $\inner{\w_t - \u}{\x_t} = \yhatt - y_t(\u)$,  $\gamma = \frac{\eta}{G^2}$, $\sigma = D^2$, and that $\|\u\|_2 \leq D$
\begin{align*}
    & \sumT (\yhatt - y_t(\u))g_t \\
    & \leq \frac{1}{2\gamma} + \frac{dG^2}{2\kappa_T \eta} \ln\left(1 + 2D^2\gamma^2\Big(\max_t\|\x_tg_t\|_2^2\Big)\frac{T}{d}\right) + \sumT \gamma \kappa_t ((\yhatt - y_t(\u))g_t)^2 \\
    & \leq \frac{G^2}{2\eta} + \frac{dG^2}{2 \kappa_T \eta} \ln\left(1 + 2D^2\eta^2\Big(\max_t\|\x_t\|_2^2\Big)\frac{T}{d}\right) + \eta \sumT \kappa_t (\yhatt - y_t(\u))^2~,
\end{align*}
where we used the assumption that $g_t^2 \leq G^2$, which completes the proof. 
\end{proof}

\section{Details of Section~\ref{sec:statisticallearning} (Statistical Learning)}\label{app:statlearn}

To directly use our results obtained in the general online setting, we set the following notation for the rest of the section:
\begin{equation}
\label{eq:shortnotation}
\ell_t(\cdot) = \ell(\cdot, Y_t), \quad y_t(f) = f(X_t), \quad \ystart = f^{\star}(X_t) \quad\textrm{and}\quad \yhatt = \hat{f}_t(X_t)~,
\end{equation}
where $\hat{f}_t$ is a statistical estimator constructed based on $(X_1, Y_1), \ldots, (X_{t-1}, Y_{t -1})$.

Let $r_t = \ell_t(\yhatt) - \ell_t(\ystart)$. To prove our high-probability bounds, we are interested in controlling $\sumT\E_{t-1}[r_t]$. As we mentioned, the challenge in obtaining high-probability bounds comes from bounding $\E_{t-1}[r_t] - r_t$, which may be of order $\sqrt{T}$ in the worst case due to the variance of $r_t$. The negative term is our regret bounds is used to control the variance term.

The following two versions of Freedman's inequality for martingales appear explicitly in \citep[Theorem 1]{beygelzimer2011contextual} and \citep[Lemma 3]{rakhlin2011making}. We use them to prove Lemmas \ref{lem:statlearnLemma} and \ref{lem:statlearnLemmaUniform}, which in turn are used to prove Theorems \ref{th:expertbatch}, \ref{th:expertbatchnonUni}, and \ref{th:otbregression}. 

\begin{lemma}[Versions of Freedman's inequality]\label{lem:bernie}
Let $X_1, \ldots, X_T$ be a martingale
difference sequence adapted to a filtration $(\mathcal{G}_i)_{i \le T}$. That is, in particular, $\E_{t-1}[X_t] = 0$.
Suppose that $|X_t| \leq R$ almost surely. Then for any $\delta \in (0,1), \lambda \in [0, 1/R]$, with probability at least $1-\delta$, it holds that
\begin{equation}
\label{eq:firstfreedmanineq} 
    \sumT X_t \leq \lambda(e-2)\sumT \E_{t-1}[X_t^2] + \frac{\ln(1/\delta)}{\lambda}.
\end{equation}
Moreover, if $\delta \in (0, 1/2), T \ge 4$, then uniformly over all $s \le T$, with probability at least $1-\delta$, it holds that
\begin{equation}\label{eq:secondfreedmanineq}
\sum_{t=1}^{s} X_t \leq 4\sqrt{\sum_{t=1}^{s} \E_{t-1}[X_t^2]\log(\log (T)/\delta)} + 2R\log(\log (T)/\delta).
\end{equation}
\end{lemma}

\begin{lemma}\label{lem:statlearnLemma}
{Under the notation \eqref{eq:shortnotation}} suppose that $\max_t \max\{|\yhatt|,|\ystart|\} \leq \half M$ {almost surely} and that  $\max_t |\ell_t'(y)| \leq M$ {almost surely} for all $y$ such that $ |y| \leq \half M$. Suppose that $\yhat_1, \ldots, \yhat_{T}$ satisfy \eqref{eq:blackbox} with $\eta = \frac{\mu}{4}$. Then, with probability at least $1 - \delta$, it holds that
\begin{align*}
    \sumT \E_{t-1}[\ell_t(\yhatt) - \ell_t(\ystart)] \leq \frac{8C_T}{\mu} + B_T + \ln(1/\delta)\min \left\{\frac{1}{(2 + \frac{\mu}{2}) M^2}, \frac{\mu}{6M^2} \left(1 + \frac{\mu^2}{16}\right)^{-1}\right\}^{-1}.
\end{align*}
\end{lemma}
\begin{proof}
Let $v_t = (\yhatt - \ystart)^2$ and let $r_t = \ell_t(\yhatt) - \ell_t(\ystart)$. By convexity and the assumptions on $\ystart, \yhatt,$ and $\ell_t'$ we have that $|\ell_t(\yhatt) - \ell_t(\ystart)| \leq |\yhatt - \ystart|M \leq  M^2$. This implies that 
\begin{align*}
    \left|r_t - \frac{\mu}{4}v_t\right| \leq |r_t| + \frac{\mu}{4}v_t \leq M^2 + \frac{\mu}{4} M^2.
\end{align*}
Thus, by equation \eqref{eq:firstfreedmanineq} we have that, for $\lambda \in \Big[0, \frac{1}{2\left(M^2 + \frac{\mu}{4} M^2\right)}\Big]$, with probability at least $1 - \delta$,
\begin{align*}
    & \sumT \E_{t-1}\left[r_t + \frac{\mu}{4}v_t\right] - \frac{\mu}{4} v_t - r_t \\
    &\qquad \leq  \lambda(e-2)\sumT \E_{t-1}\left[\left(\E_{t-1}\left[r_t + \frac{\mu}{4}v_t\right] - \frac{\mu}{4} v_t - r_t\right)^2\right] + \frac{\ln(1/\delta)}{\lambda}\\
    &\qquad \leq  \lambda(e-2)\sumT \E_{t-1}\left[\left(\frac{\mu}{4} v_t + r_t\right)^2\right] + \frac{\ln(1/\delta)}{\lambda}~. 
\end{align*}
where we used that $\E[(X - \E[X])^2] \leq \E[X^2]$.  Since $r_t^2 \leq (\ystart - \yhatt)^2 M^2$ we have that 
\begin{align*}
    \E_{t-1}\left[\left(r_t + \frac{\mu}{4}v_t\right)^2\right] &\leq 2 \E_{t-1}[r_t^2] + \frac{\mu^2}{{8}}\E[v_t^2] \\
    & \leq \E_{t-1}\left[2(\ystart - \yhatt)^2 M^2 + \frac{\mu^2M^2}{8}(\ystart - \yhatt)^2\right] 
    \\
    &= 2\left(M^2 + \frac{\mu^2M^2}{16}\right)\E_{t-1}\left[v_t\right].
\end{align*}
which means that, with probability at least $1 - \delta$,
\begin{equation} \label{eq:statlearnvarbernie}
\begin{split}
    \sumT & \E_{t-1}\left[r_t + \frac{\mu}{4}v_t\right] - \frac{\mu}{4} v_t - r_t \\
     & \leq \lambda(e-2)\sumT 2\left(M^2 + \frac{\mu^2M^2}{16}\right)\E_{t-1}\left[v_t\right] + \frac{\ln(1/\delta)}{\lambda} \\
     & \leq \lambda \sumT \frac{3}{2}\left(M^2 + \frac{\mu^2M^2}{16}\right)\E_{t-1}\left[v_t\right] + \frac{\ln(1/\delta)}{\lambda}.
\end{split}
\end{equation}
Using the guarantee on $\regret_T$ in equation~\eqref{eq:blackbox}, Lemma~\ref{lem:moVarlessR}, and replacing $\eta$ with $\frac{\mu}{4}$ we have that 
\begin{align*}
    & \sumT \E_{t-1}[r_t] \\
    & =  \regret_T +  \sumT (\E_{t-1}[r_t] - r_t) \\
    & \leq  \frac{C_T}{\eta} + \left(\eta - \frac{\mu}{2}\right) \sumT v_t + B_T + \sumT (\E_{t-1}[r_t] - r_t) \\
    & =  \frac{C_T}{\eta} + \left(\eta - \frac{\mu}{4}\right) \sumT v_t + B_T - \frac{\mu}{4}\sumT\E_{t-1}[v_t] + \sumT \left(\E_{t-1}[r_t] - r_t + \frac{\mu}{4}(\E_{t-1}[v_t] - v_t)\right) \\
    & = \frac{8C_T}{\mu} + B_T  - \frac{\mu}{4}\sumT\E_{t-1}[v_t] + \sumT \left(\E_{t-1}[r_t] - r_t + \frac{\mu}{4}(\E_{t-1}[v_t] - v_t)\right).
\end{align*}
Thus, by equation \eqref{eq:statlearnvarbernie} we have that, with probability at least $1 - \delta$,
\begin{align*}
    \sumT \E_{t-1}[r_t] \leq & \frac{8C_T}{\mu} + B_T - \frac{\mu}{4}\sumT\E_{t-1}[v_t] + \lambda \sumT \frac{3}{2}\left(M^2 + \frac{\mu^2M^2}{16}\right)\E_{t-1}\left[v_t\right] + \frac{\ln(1/\delta)}{\lambda}~.
\end{align*}
Thus, setting $\lambda = \min\left\{\frac{1}{2\left(M^2 + \frac{\mu}{4} M^2\right)}, \frac{\mu}{6} \left(M^2 + \frac{\mu^2M^2}{16}\right)^{-1}\right\}$ completes the proof.  
\end{proof}

\begin{lemma}\label{lem:statlearnLemmaUniform}
{Under the notation \eqref{eq:shortnotation}} suppose that $\max_t \max\{|\yhatt|,|\ystart|\} \leq \half M$ {almost surely} and that  $\max_t |\ell_t'(y)| \leq M$ {almost surely} for all $y$ such that $ |y| \leq \half M$. Suppose that $\yhat_1, \ldots, \yhat_{S}$ satisfy \eqref{eq:blackbox} with $\eta = \frac{\mu}{8}$. Then, for $\delta \in (0, \half)$, $T \geq 4$, and uniformly over all $S \leq T$, with probability at least $1 - \delta$, it holds that
\begin{align*}
    \sum_{t = 1}^S \E_{t-1}[\ell_t(\yhatt) - \ell_t(\ystart)] \leq \frac{8C_T}{\mu} + B_T -\frac{\mu}{8} \sum_{t=1}^S (\yhatt - \ystart)^2 + M^2\left(\frac{32}{\mu} + 3\mu + 4\right)\log\left(\frac{\log T}{\delta}\right).
\end{align*}
\end{lemma}
\begin{proof}
Let $v_t = (\yhatt - \ystart)^2$ and $r_t = \ell_t(\yhatt) - \ell_t(\ystart)$. By convexity and the assumptions on $\ystart, \yhatt,$ and $\ell_t'$ we have that $|\ell_t(\yhatt) - \ell_t(\ystart)| \leq |\yhatt - \ystart|M \leq  M^2$. This implies that 
\begin{align*}
    \left|r_t - \frac{\mu}{4}v_t\right| \leq |r_t| + \frac{\mu}{4}v_t \leq M^2 + \frac{\mu}{4} M^2.
\end{align*}
Thus, by equation~\eqref{eq:secondfreedmanineq} we have that, with probability at least $1 - \delta$
\begin{align*}
    & \sum_{t = 1}^s \E_{t-1}\left[r_t + \frac{\mu}{4}v_t\right] - \frac{\mu}{4} v_t - r_t  \\
    & \leq 4\sqrt{\sum_{t=1}^{s} \E_{t-1}\left[\left(\frac{\mu}{4} v_t + r_t\right)^2\right]\log(\log (T)/\delta)} +  M^2\left(4 + \mu\right) \log(\log (T)/\delta). 
\end{align*}
where we used that $\E[(X - \E[X])^2] \leq \E[X^2]$.  Since $r_t^2 \leq (\ystart - \yhatt)^2 M^2$ we have that 
\begin{align*}
    \E_{t-1}&\left[\left(r_t + \frac{\mu}{4}v_t\right)^2\right] \leq 2 \E_{t-1}[r_t^2] + \frac{2\mu^2}{16}\E[v_t^2] \\
    & \leq \E_{t-1}\left[2(\ystart - \yhatt)^2 M^2 + \frac{2\mu^2M^2}{16}(\ystart - \yhatt)^2\right] = 2\left(M^2 + \frac{\mu^2M^2}{16}\right)\E_{t-1}\left[v_t\right].
\end{align*}
which means that with probability at least $1 - \delta$
\begin{equation} \label{eq:statlearnvarbernie2}
\begin{split}
    & \sum_{t = 1}^s \E_{t-1}\left[r_t + \frac{\mu}{4}v_t\right] - \frac{\mu}{4} v_t - r_t \\
    & \leq  4\sqrt{\sum_{t=1}^{s} 2\left(M^2 + \frac{\mu^2M^2}{16}\right)\E_{t-1}\left[v_t\right]\log(\log (T)/\delta)} + 2R\log(\log (T)/\delta) \\
    & \leq \frac{\lambda}{2} \sum_{t=1}^{s} \E_{t-1}\left[v_t\right] + \frac{16}{\lambda}\left(M^2 + \frac{\mu^2M^2}{16}\right)\log(\log (T)/\delta) + M^2\left(4 + \mu\right)\log(\log (T)/\delta)
\end{split}
\end{equation}
for any $\lambda > 0$, where in the final inequality we used $\sqrt{ab} = \half \inf_{\eta > 0} \eta a + \frac{b}{\eta}$ for $a, b \geq 0$.
Using the guarantee on $\regret_T$ in equation~\eqref{eq:blackbox}, Lemma~\ref{lem:moVarlessR}, and replacing $\eta$ with $\frac{\mu}{8}$ we have that 
\begin{align*}
    \sum_{t=1}^S \E_{t-1}[r_t] = & \regret_T +  \sum_{t=1}^S (\E_{t-1}[r_t] - r_t) \\
    \leq & \frac{C_T}{\eta} + \left(\eta - \frac{\mu}{2}\right) \sum_{t=1}^S v_t + B_T + \sum_{t=1}^S (\E_{t-1}[r_t] - r_t) \\
    = & \frac{C_T}{\eta} + \left(\eta - \frac{\mu}{4}\right) \sum_{t=1}^S v_t + B_T - \frac{\mu}{4}\sum_{t=1}^S\E_{t-1}[v_t] 
    \\
    &\qquad+ \sum_{t=1}^S\left(\E_{t-1}[r_t] - r_t + \frac{\mu}{4}(\E_{t-1}[v_t] - v_t)\right) \\
    = & \frac{8C_T}{\mu} + B_T  -\frac{\mu}{8} \sum_{t=1}^S v_t - \frac{\mu}{4}\sum_{t=1}^S\E_{t-1}[v_t] \\
    &\qquad+ \sum_{t=1}^S \left(\E_{t-1}[r_t] - r_t + \frac{\mu}{4}(\E_{t-1}[v_t] - v_t)\right).
\end{align*}
Thus, by~\eqref{eq:statlearnvarbernie2} we have that, with probability at least $1 - \delta$,
\begin{align*}
    \sum_{t=1}^S \E_{t-1}[r_t] \leq & \frac{8C_T}{\mu} + B_T -\frac{\mu}{8} \sum_{t=1}^S v_t - \frac{\mu}{4}\sum_{t=1}^S\E_{t-1}[v_t] + \lambda \half \sum_{t=1}^{s} \E_{t-1}\left[v_t\right]  \\
    & + \frac{16}{\lambda}\left(M^2 + \frac{\mu^2M^2}{16}\right)\log(\log (T)/\delta) + M^2\left(4 + \mu\right)\log(\log (T)/\delta).
\end{align*}
Thus, setting $\lambda = \frac{\mu}{2}$ gives us 
\begin{align*}
    \sum_{t=1}^S \E_{t-1}[r_t] \leq & \frac{8C_T}{\mu} + B_T -\frac{\mu}{8} \sum_{t=1}^S v_t + M^2\left(\frac{32}{\mu} + 3\mu + 4\right)\log(\log (T)/\delta)~,
\end{align*}
which completes the proof. 
\end{proof}

The following theorem is the detailed statement of Theorem \ref{prop:expertearlystopbigO} in the main body of the paper.

\begin{theorem}\label{th:expertbatch}
Suppose that for all $f \in \Fset$ $|f(X)| \leq \half M$ almost surely, that $|\partial_y \ell(y, Y)| \leq M$ almost surely for all $y$ such that $|y| \leq \half M$, and that $\ell$ is $\mu$-strongly convex in its first argument. Then, with probability at least $1 - \delta$, Algorithm~\ref{alg:expertsonlinetobatch} with input parameters $T$, $\Sset = \frac{8M^2\ln(K)}{\mu} + M^2\left(\frac{32}{\mu} + 3\mu + 4\right)\log(\log (T)/\delta)$, $\eta = \frac{\mu}{8}$, and $M$ guarantees
\[
    \Lset(\fhat) - \min_{f \in \Fset} \Lset(f)
\le
    \left\{ \begin{array}{cl}
        0 & \text{if $S < T$,}
    \\
        {\displaystyle \frac{M^2}{\mu T} \Big(8\ln(K) + (32 + 3\mu^2 + 4\mu) \log(\log (T)/\delta)\Big) } & \text{if $S=T$.}
    \end{array} \right.
\]
\end{theorem}
\begin{proof}
Convexity of $\Lset$ gives us
\begin{align*}
    \Lset(\hat{f}) \leq \frac{1}{S}\sumTprime\Lset\left(\sumK p_t(i)f_i\right) = \frac{1}{S}\sumTprime\E_{t-1}\left[\ell_t\left(\sumK p_t(i)f_i(X_t)\right)\right]=  \frac{1}{S}\sumTprime \E_{t-1}[\ell_t(\yhatt)].
\end{align*}
Now, for any fixed $f \in \Fset$, by Lemma~\ref{lem:statlearnLemmaUniform} and Lemma~\ref{th:simplePwEAregret} we have that, with probability $1 - \delta$, simultaneously for all $S \le T$,
\begin{equation}\label{eq:proofexpertbatch}
\begin{split}
    \Lset(\hat{f}) - \Lset(f) \leq & \frac{1}{S} \sumTprime \E_{t-1}[\ell_t(\yhatt) - \ell_t(y_t(i))] \\
    \leq & \frac{\frac{8M^2\ln(K)}{\mu}  -\frac{\mu}{8} \sum_{t=1}^S (\yhatt - f(X_t))^2 + M^2\left(\frac{32}{\mu} + 3\mu + 4\right)\log(\log (T)/\delta)}{S}.
\end{split}
\end{equation}
We split the remainder of the proof into two cases. Either $S < T$ or $S = T$. If $S < T$ then $\Sset = \frac{8M^2\ln(K)}{\mu} + M^2\left(\frac{32}{\mu} + 3\mu + 4\right)\log(\log (T)/\delta) \leq \frac{\mu}{8}\min_{i}\sumTprime(\yhatt - y_t(i))$ and thus
\begin{align*}
    \Lset(\fhat) - \Lset(f) \leq 0.
\end{align*}
To complete the proof, observe that if $S = T$, then from~\eqref{eq:proofexpertbatch} we have that  
\begin{align*}
    \Lset(\fhat) - \Lset(f) \leq & \frac{\frac{8M^2\ln(K)}{\mu} + M^2\left(\frac{32}{\mu} + 3\mu + 4\right)\log(\log (T)/\delta)}{T}.
\end{align*}
The claim follows.
\end{proof}

The following result is a detailed version of Theorem~\ref{prop:expertbatchnonUni} in the main text. 

\begin{theorem}\label{th:expertbatchnonUni}
Suppose that for all $f \in \Fset$ $|f(X)| \leq \half M$ almost surely, that $|\partial_y \ell(y, Y)| \leq M$ almost surely for all $y$ such that $|y| \leq \half M$, and that $\ell$ is $\mu$-strongly convex in its first argument. Then, with probability at least $1 - \delta$, Algorithm~\ref{alg:expertsonlinetobatch} with input parameters $T$, $\eta = \frac{\mu}{4}$, $\Sset = \infty$, and $M$,
\begin{align*}
    & \Lset\left(\fhat\right) - \min_{f \in \Fset}\Lset(f) \\
    & \leq \frac{1}{T}\left(\frac{4 M^2\ln(K)}{\mu} + \ln(1/\delta)\min \left\{\frac{1}{(2 + \frac{\mu}{2}) M^2}, \frac{\mu}{6M^2} \left(1 + \frac{\mu^2}{16}\right)^{-1}\right\}^{-1}\right).
\end{align*}
\end{theorem}
\begin{proof}
Observe that $\eta < \half$ by assumption on $\mu$, making it a valid choice. Furthermore, by the choice of $\Sset$ we have that on termination of Algorithm~\ref{alg:expertsonlinetobatch} $S = T$ and thus
convexity of $\Lset$ gives us
\begin{align*}
    \Lset(\fhat) \leq \frac{1}{T}\sumT\Lset\left(\sumK p_t(i)f_i\right) = \frac{1}{T}\sumT\E_{t-1}\left[\ell_t\left(\sumK p_t(i)f_i(X_t)\right)\right]=  \frac{1}{T}\sumT \E_{t-1}[\ell_t(\yhatt)].
\end{align*}
Now, for any $f \in \Fset$, by Lemma~\ref{lem:statlearnLemma} and Lemma~\ref{th:simplePwEAregret} we have that with probability $1 - \delta$
\begin{equation*}\label{eq:proofexpertbatchnonuni}
\begin{split}
    \Lset&\left(\fhat\right) - \Lset(f) 
    \\
    \leq &\frac{1}{T} \sumT \E_{t-1}[\ell_t(\yhatt) - \ell_t(y_t(i))] \\
    \leq & \frac{1}{T}\left(\frac{4 M^2\ln(K)}{\mu} + \ln(1/\delta)\min \left\{\frac{1}{(2 + \frac{\mu}{2}) M^2}, \frac{\mu}{6M^2} \left(1 + \frac{\mu^2}{16}\right)^{-1}\right\}^{-1}\right)~,
\end{split}
\end{equation*}
which completes the proof. 
\end{proof}

The following Theorem is a detailed version of Theorem \ref{prop:otbregression}.

\begin{theorem}\label{th:otbregression}
Fix $Z > 0$. Suppose that $\ell$ is $\mu-$strongly convex in its first argument and that $\|X_t\|_2 \leq B$ and $\sup\nolimits_{y \in [-Z, Z]}|\partial\ell(y, T)| \le G$ almost surely. 
{Fix any $\w \in \mathbb{R}^d$ such that $\|\w\|_2 \leq \min\{D, Z/B\}$}, then, with probability at least $1 - \delta$,
\begin{align*}
    \Lcal\Bigg(\frac{1}{T}\sumT \w_t&\Bigg) -  \Lcal(\w) \leq \frac{1}{T}\Bigg(\frac{8G^2 + \frac{dG^2}{2} \ln\left(1 + D^2\mu^2 B^2\frac{T}{2d}\right)}{\mu} \\
    & + \ln(1/\delta)\min \left\{\frac{1}{(2 + \frac{\mu}{2}) \max \{G, 2Z\}^2}, \frac{\mu}{6\max \{G, 2Z\}^2} \left(1 + \frac{\mu^2}{16}\right)^{-1}\right\}^{-1}\Bigg),
\end{align*}
where $\w_t$ are given by Algorithm~\ref{alg:onlineregression} with $\eta = \frac{\mu}{4}$, $\sigma = D^2$, $G$, $\kappa_t = 1$, and feedback $g_t = \ell_t'(\inner{\w_t}{X_t})$ for $t = 1, \ldots, T$.
\end{theorem}
\begin{proof}
Convexity of $\Lcal$ together with $\E_{t-1}[\ell_t(\yhatt)] = \Lcal(\w_t)$ gives us
\begin{align*}
    \Lcal(\bar{\w}) \leq \frac{1}{T}\sumT\Lcal(\w_t) =  \frac{1}{T}\sumT \E_{t-1}[\ell_t(\yhatt)].
\end{align*}
Let $y_t(\w) = \inner{\w}{X_t}$. Let $M = \max \{G, 2Z\}$.
{Since $\|X\| \le B$ almost surely we may assume without loss of generality that for all $\x \in \mathcal X$, it holds that $\|\x\| \le B$.}
By Lemma~\ref{lem:statlearnLemma} and Lemma~\ref{th:simpleregregret}, we have that for any fixed $\w \in \{\w: \|\w\|_2 \leq D \textnormal{~ and ~} |\inner{\w}{{\x}}| \leq Z ~ \textrm{for all} ~ \x \in \mathcal{X}\}$, with probability at least $1 - \delta$,
\begin{align*}
    & \Lcal(\bar{\w}) - \Lcal(\w) \\
    & \leq  \frac{1}{T}\sumT (\E_{t-1}[\ell_t(\yhatt)] - \E_{t-1}[\ell_t(y_t(\w))])  \\
    & \leq \frac{1}{T}\left(\frac{8C_T}{\mu} + \ln(1/\delta)\min \left\{\frac{1}{(2 + \frac{\mu}{2}) M^2}, \frac{\mu}{6M^2} \left(1 + \frac{\mu^2}{16}\right)^{-1}\right\}^{-1}\right) \\
    & = \frac{1}{T}\Bigg(\frac{8G^2 + \frac{dG^2}{2} \ln\left(1 + D^2\mu^2 B^2\frac{T}{2d}\right)}{\mu} \\
    & \quad + \ln(1/\delta)\min \left\{\frac{1}{(2 + \frac{\mu}{2}) M^2}, \frac{\mu}{6M^2} \left(1 + \frac{\mu^2}{16}\right)^{-1}\right\}^{-1}\Bigg).
\end{align*}
\end{proof}

\section{Details of Section~\ref{sec:corruption} (Corrupted Feedback)}\label{app:corruption}

We first restate Theorem~\ref{prop:expertcorruption}, after which we prove it. 

\propexpertcorrupt*
\begin{proof}
First, observe that $\eta < \half$ by assumption on $\mu$, making it a valid choice. 
Using Lemma~\ref{th:simplePwEAregret} and~\eqref{eq:elem} we obtain
\begin{align*}
    \sumT (\yhatt - \ystart)\ell_t'(\yhatt) & = \sumT (\yhatt - \ystart)g_t - \sumT(\yhatt - \ystart)c_t \\
    & \leq \frac{M^2\ln(K)}{\eta} + \sumT \left(\eta (\yhatt - y_t(i^\star))^2 + \frac{\mu}{4} (\yhatt - \ystart)^2 + \frac{c_t^2}{\mu} \right).
\end{align*}
We conclude by using Lemma~\ref{lem:moVarlessR}:
\begin{align*}
    \regret_T \leq \frac{M^2\ln(K)}{\eta} + \left(\eta + \frac{\mu}{4} - \frac{\mu}{2}\right) \sumT (\yhatt - y_t(i^\star))^2 + \frac{1}{\mu} \sumT c_t^2~,
\end{align*}
after which replacing {$\eta = \frac{\mu}{8}$} completes the proof.
\end{proof}

We now restate Theorem~\ref{th:regcorruption}, after which we prove its result. 

\propregcorruption*
\begin{proof}
As in the proof of Theorem~\ref{prop:expertcorruption},
using Lemma~\ref{th:simpleregregret} and equation~\eqref{eq:elem} we obtain
\begin{align*}
    &\sumT (\yhatt - \ystart) \ell_t'(\yhatt)
    \\
    & = \sumT (\yhatt - \ystart) g_t - \sumT(\yhatt - \ystart)c_t \\
    & \leq \frac{G^2}{2\eta} + \frac{dG^2}{2 \eta} \ln\left(1 + \half D^2 \mu^2\big(\max_t\|\x_t\|_2^2\big)\frac{T}{d}\right) +  \sumT \left(\eta (\yhatt - \ystart)^2 + \frac{c_t^2}{\mu} \right)~.
\end{align*}
We continue by using Lemma~\ref{lem:moVarlessR}: 
\begin{align*}
    \regret_T \leq \frac{G^2}{2\eta} + \frac{dG^2}{4 \eta} \ln\left(1 + \half D^2 \mu^2\big(\max_t\|\x_t\|_2^2\big)\frac{T}{d}\right) + \left(\eta + \frac{\mu}{4} - \frac{\mu}{2}\right) \sumT (\yhatt - y_t^\star)^2 + \sumT \frac{c_t^2}{\mu}~,
\end{align*}
after which replacing {$\eta = \frac{\mu}{8}$} completes the proof.
\end{proof}  

\section{Details of Section~\ref{sec:selectivesample} (Selective Sampling)}\label{app:selsample}

The following Theorem is a detailed version of Theorem \ref{prop:expertslossefficientbigO}.

\begin{theorem}\label{th:expertslossefficient}
Fix an arbitrary sequence $\ell_1,\ldots,\ell_T$ of $\mu$-strongly convex differentiable losses. Then the predictions $\yhatt$ of Algorithm~\ref{alg:PwEA} run with inputs $M \ge \max_t|\ell_t'(\yhatt)|$, $\eta = \frac{\mu}{4}$, feedback $g_t = \frac{o_t}{q_{t-1}} \ell_t'(\yhatt)$, and $\kappa_t = q_t$ satisfy
\begin{equation*}
\Eb{\sumT(\ell_t(\yhatt) - \ell_t(\ystart))}
    \leq \left(\frac{4 M^2( \ln(K) + 1)}{\mu^{3/2} \beta }\right)^2  + \frac{4M^2( \ln(K) + 1)}{\mu}~,
\end{equation*}
provided $\max_i \max_t|\yhatt - y_t(i)| \leq M$.
\end{theorem}
\begin{proof}
First observe that $q_{t-1} \geq q_t$ and thus $\kappa_t = q_t$ is a valid choice, where we define $q_0 = q_1$. By equation \eqref{eq:regretexpertsintermediate} we have 
\begin{align*}
    \sumT \gamma (\yhatt - y_t(i))\ell_t'(\yhatt) \frac{o_t}{q_{t-1}} \leq & \frac{\ln(K)}{q_T} + \gamma^2 \sumT \kappa_{t-1}(\yhatt - y_t(i))^2g_t^2\\
    \leq & \frac{\ln(K)}{q_T} + \gamma^2 \sumT \frac{o_t}{q_{t-1}} (\yhatt - y_t(i))^2\ell_t'(\yhatt)^2 \\
    \leq & \frac{\ln(K)}{q_T} + \gamma^2 \sumT \frac{o_t}{q_{t}} (\yhatt - y_t(i))^2M^2~,
\end{align*}
where in the final inequality we used $q_t \leq q_{t-1}$ and $|\ell_t'(\yhatt)| \leq M$. 
After dividing both sides of the above inequality by $\gamma = \frac{\eta}{M^2}$ we find
\begin{align}
\label{eq:clip-1}
    \sumT (\yhatt - y_t(i))\ell_t'(\yhatt) \frac{o_t}{q_{t-1}} \leq & \frac{M^2 \ln(K)}{\eta q_T} + \eta \sumT \frac{o_t}{q_{t}} (\yhatt - y_t(i))^2.
\end{align}
Following the analysis of the clipping trick by \citet{cutkosky2019artificial}, we have that
\begin{align}
\label{eq:clip-2}
    \sumT (\yhatt - y_t(i))\ell_t'(\yhatt) \left(\frac{o_t}{q_{t}} - \frac{o_t}{q_{t-1}}\right) \leq M^2 \sumT \left(\frac{1}{q_{t}} - \frac{1}{q_{t-1}}\right) \leq & \frac{M^2}{q_T},
\end{align}
where we used that the sum telescopes, that $q_t \leq q_{t-1}$, and that $q_0 = q_1$. Summing side by side~\eqref{eq:clip-1} and~\eqref{eq:clip-2} we find 
\begin{align*}
    \sumT (\yhatt - y_t(i))\ell_t'(\yhatt) \frac{o_t}{q_{t}} \leq & \frac{M^2( \ln(K) + 1)}{\eta q_T} + \eta \sumT \frac{o_t}{q_{t}} (\yhatt - y_t(i))^2.
\end{align*}
Since $\ell_t$ is $\mu$-strongly convex we have that 
\begin{align*}
    \Eb{\ell_t(\yhatt) - \ell_t(y_t(i))} \leq \Eb{(\yhatt - y_t(i))\ell_t'(\yhatt) - \frac{\mu}{2}(\yhatt - y_t(i))^2}~,
\end{align*}
and therefore 
\begin{align*}
    \Eb{\sumT (\ell_t(\yhatt) - \ell_t(y_t(i)))}
    \leq & \Eb{\frac{M^2 ( \ln(K) + 1)}{\eta q_T}} + \left(\eta - \frac{\mu}{2}\right) \sumT \Eb{(\yhatt - y_t(i))^2},
\end{align*}
where we used that $\E_{t-1}[o_t] = q_t$.
Now, by using $\min\{a, b\}^{-1} \leq a^{-1} + b^{-1}$ for $a, b > 0$ and Jensen's inequality we have that 
\begin{align*}
    \Eb{q_T^{-1}} = & \Eb{\left(\min \left\{1, \beta\left/\sqrt{\min_{i}\sumT (\yhatt - y_t(i))^2}\right.\right\}\right)^{-1}} \\
    \leq & 1 + \beta^{-1}\sqrt{\Eb{\min_{i}\sumT (\yhatt - y_t(i))^2}}~,
\end{align*}
and thus 
\begin{align*}
    & \Eb{\sumT(\ell_t(\yhatt) - \ell_t(\ystart))} \leq \frac{M^2( \ln(K) + 1)}{\eta} \\
    & + \beta^{-1}\sqrt{\Eb{\min_{i}\sumT (\yhatt - y_t(i))^2}}\frac{M^2( \ln(K) + 1)}{\eta } - \Eb{\sumT \frac{\mu}{4}(\yhatt - \ystart)^2}.
\end{align*}
Using $ab \leq \frac{a^2\mu}{4} + \frac{b^2}{\mu}$ for $a, b > 0$ we continue
\begin{align*}
    \Eb{\sumT(\ell_t(\yhatt) - \ell_t(\ystart))}
    & \leq \left(\frac{M^2( \ln(K) + 1)}{\eta \beta \sqrt{\mu}}\right)^2  + \frac{M^2( \ln(K) + 1)}{\eta}~.
\end{align*}
Using that $\eta^{-1} = \frac{4}{\mu}$ we arrive at the conclusion of the proof:
\begin{align*}
    \Eb{\sumT(\ell_t(\yhatt) - \ell_t(\ystart))}
    & \leq \left(\frac{4 M^2( \ln(K) + 1)}{\mu^{3/2} \beta }\right)^2  + \frac{4M^2( \ln(K) + 1)}{\mu}~.
\end{align*}
\end{proof}

\subsection{Selective Sampling for Online Regression}\label{sec:regressionselsampling}
We now extend our selective sampling results to online regression, where we assume that the feature vector $\x_t$ is revealed to the learner before issuing a prediction. 
We run Algorithm~\ref{alg:onlineregression} and, similarly to Section~\ref{sec:selectivesample}, we request the loss at round $t$ by drawing a Bernoulli variable $o_t$ of parameter
\begin{align}\label{eq:regressionqt}
    q_t = \min \left\{1, \beta\left/\min_{\w \in \domainw_t \cap \{\w: \|\w\|_2 \leq D\}} \sumt (\yhat_s - \inner{\w}{\x_s})^2\right. \right\}~,
\end{align}
for some $D,\beta > 0$. This gives the following expected regret guarantee.
\begin{theorem} 
Fix an arbitrary sequence $\ell_1,\ldots,\ell_T$ of $\mu$-strongly convex differentiable losses. Then the predictions $\yhatt$ of Algorithm~\ref{alg:onlineregression} run with inputs $\eta = \frac{\mu}{4}$, $\sigma = D^2$, $G \ge \max_t|\ell_t'(\yhatt)|$, $Z = \half G$, feedback $g_t = \frac{o_t}{q_t}\ell_t'(\yhatt)$ and $\kappa_t = q_t$, satisfy
\begin{align*}
   \Eb{\sumT(\ell_t(\yhatt) - \ell_t(\ystart))}
    & \leq \left(\frac{2 \zeta_T}{\mu^{3/2} \beta }\right)^2  + \frac{2 \zeta_T}{\mu}.
\end{align*}
for any $\x_1,\ldots,\x_T\in\reals^d$, and for any $\u \in \domainw_T \equiv \bigcap_{t = 1}^T \{\w: |\inner{\w}{\x_t}| \leq Z\}$ such that $\|\u\|_2 \leq D$ and $\ystart = \inner{\u}{\x_t}$ for all $t \ge 1$, 
where 
\begin{align*}
    \zeta_T = G^2 + dG^2\ln\left(1 + D^2 \max_t\|\x_t\|_2^2\left(\frac{T}{2 d} + \frac{G^2 T^2}{2 d \beta^2} \right)\right).
\end{align*}
\end{theorem}
\begin{proof}
Starting from \eqref{eq:regregretinter} and replacing $g_t$ by $\frac{o_t}{q_t} \ell_t'(\yhatt)$ we find
\begin{align*}
    & \sumT \frac{o_t}{q_t}\inner{\w_t - \u}{\x_t\ell_t'(\yhatt)} \\
    & \leq \frac{\|\u\|_2^2}{2\sigma \gamma } + \frac{d}{2 \kappa_T \gamma} \ln\left(1 + 2\sigma\gamma^2\Big(\max_t\|\x_t g_t\|_2^2\Big)\frac{T}{d}\right) +  \gamma \sumT \kappa_t \frac{o_t}{q_t^2}(\inner{\w_t - \u}{\x_t\ell_t'(\yhatt)})^2 \\
    & \leq \frac{G^2}{2 \eta} + \frac{d G^2}{2 q_T \eta} \ln\left(1 + D^2 q_t^{-2}\Big(\max_t\|\x_t \|_2^2\Big)\frac{T}{2d}\right) +  \gamma \sumT \frac{o_t}{q_t}(\inner{\w_t - \u}{\x_t})^2~,
\end{align*}
where in the second inequality we used that $\kappa_t = q_t$, $\|\u\|_2 \leq D$, $\sigma^2 = D^2$, $\gamma = \frac{\eta}{G^2}$, $\eta \leq \half$ by assumption on $\mu$, and $|\ell_t'(\yhatt)| \leq G$.
Now, using that $\inner{\w_t - \u}{\x_t\ell_t'(\yhatt)} = (\yhatt - \ystart)\ell_t'(\yhatt)$ and by taking the expectation of both sides of the above and using that $\E_{t-1}[o_t] = q_t$ we find 
\begin{align*}
    & \Eb{\sumT (\yhatt - \ystart)\ell_t'(\yhatt)} \\
    & \leq \Eb{\frac{G^2}{2q_T \eta }} + \Eb{\frac{d G^2}{2 \eta q_T} \ln\left(1 + q_T^{-2}D^2 \max_t\|\x_t\|_2^2\frac{T}{2 d}\right)} + \Eb{\sumT \eta(\yhatt - \ystart)^2} \\
    & \leq \Eb{\frac{G^2}{2q_T \eta }} + \Eb{\frac{d G^2}{2 \eta q_T} \ln\left(1 + D^2 \max_t\|\x_t\|_2^2\left(\frac{T}{2 d} + \frac{G^2 T^2}{2 d \beta^2} \right)\right)} + \Eb{\sumT \eta(\yhatt - \ystart)^2} \\
    & = \Eb{\frac{\zeta_T}{2 q_T \eta }} + \Eb{\sumT \eta(\yhatt - \ystart)^2}~,
\end{align*}
where in the final inequality we used that $q_T^{-1} \leq 1 + G\sqrt{T} / \beta$ and defined 
\begin{align*}
    \zeta_T = G^2 + dG^2\ln\left(1 + D^2 \max_t\|\x_t\|_2^2\left(\frac{T}{2 d} + \frac{G^2 T^2}{2 d \beta^2} \right)\right).
\end{align*}
Since $\ell_t$ is $\mu$-strongly convex we have that 
\begin{align*}
    \Eb{\ell_t(\yhatt) - \ell_t(\ystart)} \leq \Eb{(\yhatt - \ystart)\ell_t'(\yhatt) - \frac{\mu}{2}(\yhatt - \ystart)^2}~,
\end{align*}
and therefore 
\begin{align*}
    \Eb{\sumT(\ell_t(\yhatt) - \ell_t(\ystart))}
    & \leq \Eb{\frac{\zeta_T}{2q_T \eta }} + \Eb{\sumT \left(\eta - \frac{\mu}{2}\right)(\yhatt - \ystart)^2} \\
    & \leq \Eb{\frac{\zeta_T}{2 q_T \eta }} - \Eb{\sumT \frac{\mu}{4}(\yhatt - \ystart)^2}~,
\end{align*}
where we used that $\eta = \frac{\mu}{4}$.
Now, by using $\min\{a, b\}^{-1} \leq a^{-1} + b^{-1}$ for $a, b > 0$ and Jensen's inequality we have that 
\begin{align*}
    \Eb{q_T^{-1}} = & \Eb{\left(\min \left\{1, \beta\left/\sqrt{\min_{\w \in \domainw_T \cap \{\w: \|\w\|_2 \leq D\}}\sumT (\yhatt - \inner{\w}{\x_t})^2}\right.\right\}\right)^{-1}} \\
    \leq & 1 + \beta^{-1}\sqrt{\Eb{\min_{\w \in \domainw_T \cap \{\w: \|\w\|_2 \leq D\}}\sumT (\yhatt - \inner{\w}{\x_t})^2}}~, 
\end{align*}
and thus 
\begin{align*}
    & \Eb{\sumT(\ell_t(\yhatt) - \ell_t(\ystart))} \\
    & \leq \beta^{-1}\sqrt{\Eb{\min_{\w \in \domainw_T \cap \{\w: \|\w\|_2 \leq D\}}\sumT (\yhatt - \inner{\w}{\x_t})^2}}\frac{\zeta_T}{2 \eta } + \frac{\zeta_T}{2 \eta } - \Eb{\sumT \frac{\mu}{4}(\yhatt - \ystart)^2}.
\end{align*}
Using $ab \leq \frac{a^2\mu}{4} + \frac{b^2}{\mu}$ for $a, b > 0$ we continue
\begin{align*}
    \Eb{\sumT(\ell_t(\yhatt) - \ell_t(\ystart))}
    & \leq \left(\frac{\zeta_T}{2 \eta \beta \sqrt{\mu}}\right)^2  + \frac{\zeta_T}{2 \eta}.
\end{align*}
Using that $\eta^{-1} = \frac{4}{\mu}$ we arrive at the conclusion of the proof:
\begin{align*}
     \Eb{\sumT(\ell_t(\yhatt) - \ell_t(\ystart))}
    & \leq \left(\frac{2 \zeta_T}{\mu^{3/2} \beta }\right)^2  + \frac{2 \zeta_T}{\mu}.
\end{align*}
\end{proof}

\section{Exponentially Weighted Average}\label{sec:EW}
Here we provide a brief description of the Exponentially Weighted Average (EWA) algorithm \citep{vovk1990aggregating, littlestone1994weighted} on a discrete set of experts. EWA maintains a distribution $\p_t$ over the experts, where the mass on expert $i$ is given by
\begin{align*}
    p_t(i) \propto \exp(-\eta\sum_{s = 1}^{t-1}\ell_s(y_s(i))),
\end{align*}
where $\eta > 0$ is the learning rate. 
A standard result is that the regret of EWA can be bounded as---see, for example, \citep[Lemma 1]{vanderhoeven2018many}:
\begin{align}\label{eq:EWbound}
    \regret_T \leq \frac{\ln(K)}{\eta} + \sumT\left(\ell_t(\yhatt) + \frac{1}{\eta} \ln(\E_{i \sim \p_t}\left[\exp(-\eta \ell_t(y_t(i)))\right]) \right)~.
\end{align}
For $\alpha$-exp concave losses, which are losses for which $g_t(y) = \exp(-\alpha \ell_t(y))$ is concave, we can further bound \eqref{eq:EWbound} by choosing $\eta = \alpha$ and using Jensen's inequality:
\begin{align*}
    \regret_T \leq \frac{\ln(K)}{\alpha}~.
\end{align*}
For $\mu$-strongly convex and $G$-Lipschitz losses we recover the optimal rate by using the fact that $\mu$-strongly convex losses are $\frac{\mu}{G^2}$-exp concave \citep[Proposition 1.2]{bubeck2011introduction}:
\begin{align*}
    \regret_T \leq \frac{G^2\ln(K)}{\mu}~. 
\end{align*}

\section{High Probability Regret Bounds for Online Prediction with Abstention}\label{sec:abstention}

\begin{algorithm}[t]
\caption{AdaHedge with abstention \citep{VanderHoeven2020exploiting}}\label{alg:AdaAb}
\Input{AdaHedge} \;
\For{$t = 1, \ldots, T$}{
    Receive expert predictions $y_t(1), \ldots, y_t(K)$ \\
    Obtain distribution $\p_t$ from AdaHedge \\
    Set $\yhatt = \sumK p_t(i)y_t(i)$ \\
    Set $\tilde{y}_t = \sign(\yhatt)$ \\
    Set $b_t = 1 - |\yhatt|$ \\
    Set sample $a_t$ from a Bernoulli distribution with parameter $1 - b_t$ \\
    If $a_t = 1$, predict $\tilde{y}_t$, otherwise abstain from prediction \\
    Receive $y_t$, send $\half(1 - y_t(i)y_t)$ as the loss of the $i$-th expert to AdaHedge
}\;
\end{algorithm}  

We consider the following generalization of the online learning with abstention setting due to \citet{VanderHoeven2020exploiting}. In each round $t = 1, \ldots, T$ the learner receives expert predictions $y_t(i) \in [-1, 1]$ and the learner can then either predict $\tilde{y}_t \in [-1, 1]$ or abstain from prediction. If the learner predicts with $\tilde{y}_t$ the learner suffers half the hinge loss $\ell_t(y) = \half(1 - y y_t)$, where $y_t \in \{-1, 1\}$. If the learners abstains from prediction the learner suffers abstentions cost $\rho \in [0, \half)$. Let $a_t = 1$ if the learner predicts with $\tilde{y}_t$ and let $a_t = 0$ if the learner abstains from prediction. The goal is to control the following definition of regret:
\begin{align*}
    \sumT a_t \ell_t(\tilde{y}_t) + (1 - a_t) \rho - \sumT \ell_t(\ystart)~,
\end{align*}
where $\ystart = y_t(i^\star)$ and $i^\star = \argmin_i \sumT \ell_t(y_t(i))$. In the online prediction with abstention setting an analog of Lemma~\ref{lem:moVarlessR} can be derived. However, in the technical part we instead use Algorithm 2 by \citet{VanderHoeven2020exploiting}. Algorithm 2 by \citet{VanderHoeven2020exploiting} samples $a_t = 1$ with probability $1 - b_t$ and $a_t = 0$ with probability $b_t$, where $b_t = 1 - |\yhatt|$ and $\yhatt = \sumK p_t(i)y_t(i)$. Distributions $\p_t$ come from AdaHedge \citep{derooij2014follow} and if $a_t = 1$ we predict with $\tilde{y}_t = \sign(\yhatt)$. Algorithm 2 by \citet{VanderHoeven2020exploiting}, or Algorithm~\ref{alg:AdaAb} in this paper, has the following expected regret guarantee. 

\begin{lemma}\label{lem:abstentionnegative}
For any $\eta > 0$ Algorithm~\ref{alg:AdaAb} guarantees 
\begin{align*}
    \sumT  & ((1 - b_t)\ell_t(\tilde{y}_t) + b_t \rho) \\
    & \leq \sumT \ell_t(\ystart) + \frac{\ln(K)}{\eta} + \frac{4}{3}\ln(K) + 2 + \sumT \half\left(\eta - (1 - 2 \rho)\right)(1 - |\yhatt|)~.
\end{align*}
\end{lemma}
\begin{proof}
From \citep[Lemma 3]{VanderHoeven2020exploiting} we have that for any $\eta > 0$
\begin{align*}
    & \sumT  ((1 - b_t)\ell_t(\sign(\yhatt)) + b_t \rho) \leq \sumT \ell_t(\ystart)  + \frac{4}{3} \ln(K) + 2 \\
    &  + \frac{\ln(K)}{\eta} + \eta \sumT \big(\E_{i \sim \p_t}\left[(\ell_t(\yhatt) - \ell_t(y_t(i)))^2\right] + ((1 - b_t)\ell_t(\sign(\yhatt)) + b_t ) - \ell_t(\yhatt)\big)~.
\end{align*}
Now, by \citep[equation (16)]{VanderHoeven2020exploiting} we have that 
\begin{align*}
    \E_{i \sim \p_t}&\left[(\ell_t(\yhatt) - \ell_t(y_t(i)))^2\right] + ((1 - b_t)\ell_t(\sign(\yhatt)) + b_t ) - \ell_t(\yhatt)  \\
    & \leq \rho(1 - |\yhatt|) + \eta \half (1 - |\yhatt|) - \half (1 - |\yhatt|)~,
\end{align*}
and thus 
\begin{align*}
    \sumT & ((1 - b_t)\ell_t(\sign(\yhatt)) + b_t \rho) \\
    & \leq \sumT \ell_t(\ystart) + \frac{\ln(K)}{\eta} + \frac{4}{3} \ln(K) + 2 + \sumT \left(\eta \half - \half(1 - 2 \rho)\right)(1 - |\yhatt|),
\end{align*}
which completes the proof.
\end{proof}

To see why Lemma~\ref{lem:abstentionnegative} is the analog of Lemma~\ref{lem:moVarlessR} for the online prediction with abstention setting observe that by choosing $\eta < 1 - 2 \rho$ we recover a bound akin to \eqref{eq:motivatingbound}. In particular, by using Lemma~\ref{lem:abstentionnegative}, choosing $\eta = \half(1 - 2 \rho)$ we find
\begin{align*}
    \sumT  ((1 - b_t) \ell_t(\sign(\yhatt)) + b_t \rho) \leq & \sumT \ell_t(\ystart) + \frac{2\ln(K)}{(1 - 2 \rho)}  + \frac{4}{3} \ln(K) + 2 \\
    & - \frac{(1 - 2 \rho)}{4} \sumT (1 - |\yhatt|).
\end{align*}
meaning we can exploit the negative $\frac{(1 - 2 \rho)}{4} \sumT (1 - |\yhatt|)$ in online prediction with abstention in a similar manner as we exploited the negative variance term in online learning with strongly convex losses. As an application of Lemma~\ref{lem:abstentionnegative} we provide a high-probability bound for online prediction with abstention. 
\begin{theorem}\label{th:abstentionHP}
For $\delta \in (0, 1)$, with probability at least $1 - \delta$, Algorithm~\ref{alg:AdaAb} guarantees 
\begin{align*}
    \sumT (a_t\ell_t(\sign(\yhatt)) + (1-a_t) \rho) \leq & \frac{2\ln(K)}{1 - 2 \rho} + \frac{21\ln(1/\delta)}{8(1 - 2 \rho)}  + \frac{4}{3} \ln(K) + 2.
\end{align*}
\end{theorem}
\begin{proof}
Let $r_t = a_t\ell_t(\sign(\yhatt)) + (1-a_t) \rho - ((1 - b_t)\ell_t(\sign(\yhatt)) + b_t \rho)$. Since $\E[a_t\ell_t(\sign(\yhatt)) + (1-a_t) \rho] = ((1 - b_t)\ell_t(\sign(\yhatt)) + b_t \rho)$ and $|r_t| \leq 1$, by Lemma~\ref{lem:bernie}, for $\delta \in (0, 1)$ and $\lambda \in [0, 1]$, with probability at least $1 - \delta$ we have that 
\begin{align*}
    r_t \leq \frac{\ln(1/\delta)}{\lambda} + \lambda \frac{3}{4}\sumT \E[r_t^2] 
\end{align*}
Now, let us study $\E[r_t^2] = \E[(a_t\ell_t(\sign(\yhatt)) + (1-a_t) \rho)^2] - ((1 - b_t)\ell_t(\sign(\yhatt)) + b_t \rho)^2$. If $\sign (\yhatt) = y_t$ then 
\begin{align*}
    \E[r_t^2] = b_t \rho^2 -(b_t \rho)^2 \leq \frac{1}{4}b_t = \frac{1}{4}(1 - |\yhatt|).
\end{align*}
If $\sign (\yhatt) \not = y_t$ then 
\begin{align*}
    \E[r_t^2] = & (1 - b_t) + b_t \rho^2 -(1 - b_t + b_t \rho)^2 \\
    = & (1 - b_t) + b_t \rho^2 - b_t^2 \rho^2 - (1 - b_t)^2 + (1 - b_t)b_t \rho \\
    = & |\yhatt| - |\yhatt|^2 + (1 - |\yhatt|)\rho^2 -(1 - |\yhatt|)^2\rho^2 + |\yhatt| (1 - |\yhatt|)\rho \\
    \leq & |\yhatt|(1 - |\yhatt|) + (1 - |\yhatt|)\rho^2 \\
    \leq & \frac{7}{4}(1 - |\yhatt|),
\end{align*}
thus we may conclude that $\E[r_t^2] \leq \frac{7}{4}(1 - |\yhatt|)$.
Combining the above with Lemma~\ref{lem:abstentionnegative} we have that, with probability at least $1 - \delta$,
\begin{align*}
    \sumT (a_t\ell_t(\sign(\yhatt)) + (1-a_t) \rho) \leq & \frac{\ln(K)}{\eta} + \frac{\ln(1/\delta)}{\lambda} + \frac{4}{3} \ln(K) + 2 \\
    & + \sumT (\eta \half + \lambda \tfrac{21}{16} - \half(1 - 2\rho)). 
\end{align*}
Since $ \rho < 1/2$, $\eta = \half (1 - 2  \rho)$ and $\lambda = \frac{8}{21} (1 - 2  \rho)$ are valid choices and we obtain 
\begin{align*}
    \sumT (a_t\ell_t(\sign(\yhatt)) + (1-a_t) \rho) 
    \leq & \frac{2\ln(K)}{1 - 2 \rho} + \frac{21\ln(1/\delta)}{8(1 - 2 \rho)} + \frac{4}{3} \ln(K) + 2,
\end{align*}
which completes the proof.
\end{proof}

\section{Free Restarts}\label{sec:freerestarts}
\begin{algorithm}
\caption{Restarting}\label{alg:restartexperts}
\Input{$\eta > 0$, $M > 0$} \;
\Init{$\tau_{1} = 1$, $\nu = 1$, Algorithm~\ref{alg:PwEA} with inputs $M$ and $\eta$} \;
\For{$t = 1, \ldots, T$}{
    Receive expert predictions $y_t(1), \ldots, y_t(K)$ \\
    Send expert predictions to Algorithm~\ref{alg:PwEA} and receive $\yhatt$ \\
    Predict $\yhatt$ and receive loss $\ell_t$ \\
    Send $g_t = \ell_t'(\yhatt)$ and $\kappa_t = 1$ to Algorithm~\ref{alg:PwEA} \\
    \If{${\displaystyle \min_i \frac{\mu}{4}\sum_{s = \tau_\nu}^t (\yhat_s - y_s(i))^2 \geq \frac{4M^2\log(K)}{\mu} }$}{
        Restart Algorithm~\ref{alg:PwEA} with inputs $M$ and $\eta$ \\
        Set $\nu = \nu + 1$ and $\tau_{\nu + 1} = t + 1$
    }\;
}\;
Set $\tau_{\nu + 1} = T + 1$
\end{algorithm}  

In this section, we introduce another way to exploit negative regret. The idea is the following. We keep track of $\min_i \sum_{s = 1}^t (\yhat_s - y_s(i))^2$, which is a lower bound on $\sum_{s = 1}^t (\yhat_s - \ystar_s)^2$, and as soon as $\frac{\mu}{4}\min_i \sum_{s = 1}^t (\yhat_s - y_s(i))^2 \geq \frac{4M^2\log(K)}{\mu}$, Lemma~\ref{lem:moVarlessR} with an appropriate $\eta$ ensures that $R_t \leq 0$. This implies that we may restart the algorithm for free, and compete with a new best expert from that point on. This approach leads to a simplified dynamic regret bound---see, for example, \citep{zhang2018dynamic} or the references therein for a discussion of dynamic regret---in which the expert we are competing against may change in all rounds where the algorithm restarts. Our simplified dynamic regret bound is never larger than the standard regret bound. We denote by $\nu$ the number of restarts and by $\tau_\nu$ the first round of restart $\nu$. The algorithm can be found in Algorithm~\ref{alg:restartexperts} and its regret guarantee can be found in Theorem~\ref{th:restarting} below. 

\begin{theorem}\label{th:restarting}
Fix an arbitrary sequence $\ell_1,\ldots,\ell_T$ of $\mu$-strongly convex differentiable losses. Then the predictions $\yhatt$ of Algorithm~\ref{alg:restartexperts} run with $g_t = \ell_t'(\yhatt)$, $\kappa_t = 1$, and inputs $\eta = \frac{\mu}{4}$ and $M$ guarantees
\begin{align*}
    \sum_{n = 1}^{\nu} \max_i \sum_{t = \tau_n}^{\tau_{n + 1} - 1} \Big(\ell_t(\yhatt) - \ell_t\big(y_t(i)\big)\Big) \leq \frac{4 M^2\log(K)}{\mu}~,
\end{align*}
provided that $\max_t \max_i |y_t(i)| \leq \half M$ and $ \max_t |\ell_t'(\yhatt)| \leq M$.
\end{theorem}
\begin{proof}
First, observe that $\eta < \half$ by assumption on $\mu$, making it a valid choice for Algorithm~\ref{alg:PwEA}. 
For any $n < \nu$ we have that for any $i \in [K]$, by Lemma~\ref{th:simplePwEAregret} and Lemma~\ref{lem:moVarlessR} 
\begin{align*}
    \sum_{t = \tau_n}^{\tau_{n + 1} - 1} \Big(\ell_t(\yhatt) - \ell_t\big(y_t(i)\big)\Big)
\leq
    \frac{4 M^2\log(K)}{\mu} - \frac{\mu}{4}\sum_{t = \tau_n}^{\tau_{n + 1} - 1} \big(\yhatt - y_t(i)\big)^2.
\end{align*}
Since $n < \nu $ we must have that 
\begin{align*}
    \frac{4 M^2\log(K)}{\mu}
\leq
    \min_j \frac{\mu}{4}\sum_{t = \tau_n}^{\tau_{n + 1} - 1} \big(\yhatt - y_t(j)\big)^2
\leq
    \frac{\mu}{4}\sum_{t = \tau_n}^{\tau_{n + 1} - 1} \big(\yhatt - y_t(i)\big)^2~,
\end{align*}
and thus
\begin{align*}
    \sum_{t = \tau_n}^{\tau_{n + 1} - 1} \Big(\ell_t(\yhatt) - \ell_t\big(y_t(i)\big)\Big) \leq 0~.
\end{align*}
For $n = \nu$ we have that for any $i \in [K]$, by~Lemma~\ref{th:simplePwEAregret} and Lemma~\ref{lem:moVarlessR}
\begin{align*}
    \sum_{t = \tau_n}^{\tau_{n + 1} - 1} \left(\ell_t(\yhatt) - \ell_t(y_t(i))\right)
\leq
    \frac{4 M^2\log(K)}{\mu}~,
\end{align*}
which completes the proof. 
\end{proof}

\end{document}